\documentclass[sigconf]{acmart}
\usepackage{algorithmic}
\usepackage{algorithm}
\usepackage{multirow}
\usepackage{makecell}
\usepackage{stackengine}
\theoremstyle{remark}
\usepackage{wrapfig}
\newtheorem{remark}{Remark}

\AtBeginDocument{%
  \providecommand\BibTeX{{%
    \normalfont B\kern-0.5em{\scshape i\kern-0.25em b}\kern-0.8em\TeX}}}

\setcopyright{acmlicensed}
\copyrightyear{2018}
\acmYear{2018}
\acmDOI{XXXXXXX.XXXXXXX}

\acmConference[Conference acronym 'XX]{Make sure to enter the correct
  conference title from your rights confirmation emai}{June 03--05,
  2018}{Woodstock, NY}
\acmISBN{978-1-4503-XXXX-X/18/06}

\begin{document}

\title{
Permutation Equivariant 
Model-based Offline Reinforcement Learning for Auto-bidding}

\author{Zhiyu Mou}
\email{mouzhiyu.mzy@taobao.com}
\affiliation{%
  \institution{Alibaba Group}
  \city{Beijing}
  \country{China}
}

\author{Miao Xu}
\email{xumiao.xm@taobao.com}
\affiliation{%
  \institution{Alibaba Group}
  \city{Beijing}
  \country{China}}

\author{Wei Chen}
\email{ganhai.cw@taobao.com}
\affiliation{%
  \institution{Alibaba Group}
  \city{Beijing}
  \country{China}
}

\author{Rongquan Bai}
\email{rongquan.br@taobao.com}
\affiliation{%
 \institution{Alibaba Group}
 \city{Beijing}
 \country{China}
 }

\author{Chuan Yu}
\email{yuchuan.yc@alibaba-inc.com}
\affiliation{%
  \institution{Alibaba Group}
  \city{Beijing}
  \country{China}}

\author{Jian Xu}
\email{xiyu.xj@alibaba-inc.com}
\affiliation{%
  \institution{Alibaba Group}
  \city{Beijing}
  \country{China}
}

\renewcommand{\shortauthors}{Trovato and Tobin, et al.}

\begin{abstract}
 Recently, reinforcement learning (RL)-based auto-bidding technique has experienced a paradigm shift: from training policies with the \emph{imaginary data} generated by a simplistic offline simulator of the online advertising system, which we term the \emph{Simulation-based RL Bidding} (SRLB), to training policies with
 a fixed set of \emph{real data} collected directly from the online advertising system with offline RL methods, which we term the \emph{Offline RL Bidding} (ORLB).
 However, the trajectories of the policies trained by the ORLB are often restricted to the vicinity of the real data, and
 as the real data typically covers a limited state space, the trained policy exhibits only modest performance improvements.
 Although the imaginary data in the SRLB can expand the covered state space, it comes with a risk: it can seriously mislead the policy due to the significant gap between the simplistic offline simulator and the online advertising system. In this paper, we introduce a \emph{Model-based RL Bidding} (MRLB) paradigm that holds promise for addressing the challenges present in the SRLB and the ORLB. Specifically, the MRLB first learns an environment model through supervised learning based on real data to serve as the offline simulator, which can bridge the gap encountered in the SRLB. Then the MRLB trains the policy with both real and imaginary data, which can expand the covered state space compared to that of the ORLB. 
 As the performance of the MRLB heavily relies on the accuracy of the environment model, we propose two novel designs to further improve the reliability of the environment model for policy training.
 Firstly,
 we design the environment model as a \emph{permutation equivariant} neural network to enhance its generalization ability with theoretical guarantees. 
 Secondly, we design a robust offline Q learning method that pessimistically penalizes the reward and state predictions of the environment model to alleviate the negative impact of potential overestimations on policy training.
 These two designs constitute a specific algorithm based on the MRLB, named \emph{Permutation Equivariant Model-based Offline RL} (PE-MORL).
 Extensive real-world experiments validate the superiority of the proposed PE-MORL algorithm compared to the state-of-the-art auto-bidding algorithms.
\end{abstract}

\ccsdesc[500]{Information systems~Computational advertising}
\ccsdesc[500]{Theory of computation~Reinforcement learning}

\keywords{Auto-bidding, Reinforcement Learning, Permutation Equivariance}


\received{20 February 2007}
\received[revised]{12 March 2009}
\received[accepted]{5 June 2009}

\maketitle

\section{Introduction}
\label{sec:intro}

Auto-bidding has become a popular tool to bid for impression opportunities on behalf of advertisers in online advertising systems.
As auto-bidding inherently involves the sequential decision-making problem, reinforcement learning (RL) \cite{sutton1999policy} is playing an increasingly important role in acquiring outstanding auto-bidding policies.
In recent years, the RL-based auto-bidding technique has experienced a paradigm shift in policy training. The initial paradigm mainly focuses on training auto-bidding policies in a simplistic,
manually-constructed offline simulator of the online advertising system, which we term the \emph{Simulation-based RL Bidding} (SRLB) paradigm \cite{he2021unified,hao2020dynamic,wu2018budget,zhao2018deep,jin2018real,wen2022cooperative}, since directly training in the online advertising system is typically not allowed due to safety concerns.
However, due to the commonly existing gap between the offline simulator and the online advertising system, the data generated by the offline simulator --- referred to as the \emph{imaginary data} --- is typically biased, potentially resulting in sub-optimal or even poor auto-bidding policies in the online advertising system. 
This is known as the \emph{inconsistencies between online and offline} (IBOO) problem \cite{mou2022sustainable}. 
To tackle this problem, the current focus has shifted away from training policies with the offline simulator to directly using the fixed dataset collected from the online advertising system --- referred to as the \emph{real data} --- to train the auto-bidding policy with model-free offline RL methods, which we term the \emph{Model-free Offline RL Bidding} (ORLB) paradigm \cite{kumar2020conservative,ashvin2020accelerating,mou2022sustainable,korenkevych2023offline}.
However, the trajectories of the auto-bidding policy trained by the ORLB are often constrained to the vicinity of the real data, typically covering a limited state space \cite{mou2022sustainable}, which only results in modest improvements in the performance of the auto-bidding policy.

Essentially, both paradigms fail to fully exploit the information that could be useful for training the auto-bidding policy.
Specifically, (1) for ORLB: the information loss occurs in the portions of the state space that are not covered by the real data. 
Actually, this missing information can be supplemented by generating additional data through further interactions with an offline simulator;
(2) for SRLB: although the imaginary data generated by the offline simulator can expand the covered state space, existing offline simulators are generally constructed with the simplistic General Second Principle (GSP) rule. This approach overlooks much of the useful and available information about the bidding environment that could be contained in the real data (or could be derived from professional expertise), and as previously mentioned, these offline simulators
exhibit serious IBOO problems. 
In fact, we can employ more effective methods to construct the offline simulator with more useful information, which can largely mitigate the IBOO problem. 

Based on these analyses, here we introduce the \emph{Model-based RL Bidding} (MRLB) paradigm to the auto-bidding field, as described in Algorithm \ref{alg:mrlb}. The MRLB paradigm is mainly based on the single-agent model-based RL algorithm, which has been proven to be effective in many other tasks, such as robotics \cite{polydoros2017survey} and video games \cite{kaiser2019model}.
Specifically, the MRLB first learns an environment model through supervised learning based on real data, which serves as the  
offline simulator. 
This can increase the amount of information contained in the offline simulator compared to the GSP-based offline simulator in the SRLB and largely mitigate the IBOO problem.
Then the MRLB trains the auto-bidding policy with both the real data and the imaginary data, which expands the covered state space compared to that of the ORLB. 
Hence, the MRLB acts as a promising paradigm to compensate for the information losses in both the SRLB and the ORLB. 

However, the vanilla MRLB algorithm that directly follows Algorithm \ref{alg:mrlb} without any particular designs is still not enough to train outstanding auto-bidding policies due to the prohibition of additional interactions with online advertising systems during training and the fixed nature of the real data.
Specifically, we note that the performance of the MRLB heavily relies on the accuracy of the environment model.
Since the fixed real data typically covers a limited state space \cite{mou2022sustainable}, the environment model may be inaccurate for states outside the real data. 
 Such inaccuracies, particularly in overestimating rewards, can mislead the policy and significantly impair the policy's performance.
This is typically known as the \emph{Out-Of-Distribution} (OOD) challenge \cite{korenkevych2023offline,mou2022sustainable,yu2020mopo}.

In this paper, we propose a specific algorithm of the MRLB paradigm, named the \emph{Permutation Equivariant Model-based Offline RL} (PE-MORL), to address the OOD challenge. 
The PE-MORL features two novel designs compared to the vanilla MRLB algorithm.
\begin{itemize}
	\item Firstly, 
	we design the environment model as a \emph{permutation equivariant} model, which can have better generalization ability and accuracy both theoretically and empirically. 
	Specifically, 
	rather than only using the features of a single agent (i.e., advertiser in this paper) as inputs,  which is a typical design in single-agent model-based RL algorithms, our environment model incorporates both the states and actions of the considered advertiser and those of all the competing advertisers as inputs, and outputs the predictions on the next states of all advertisers and the reward of the considered advertiser.
	Note that the incorporation of all advertisers' information  can naturally enhance the generalization ability of the environment model. 
	The "\emph{permutation equivariant}" means the property that if the order of advertisers in the inputs to the environment model changes, then the output values will remain the same but their orders will permute accordingly. An insight behind this design is that the order of inputs does not alter the information they contain, and therefore, only the order of the output values should change accordingly \cite{lee2019set}.
	Theoretically, we can prove that making the environment model a permutation equivariant function can increase its generalization ability. 
	\item 
	Secondly, we design a robust offline Q learning method to deal with the uncertainty of the states and actions the environment predicts, which can further address the OOD challenge.
	Specifically, in RL algorithms, the vanilla Q loss function is only a Bellman error term (as introduced in Section \ref{sec:pre}), which can be easily influenced by the overestimation of the environment model, making the trained auto-bidding policy improperly deviate from the real data.
In the proposed robust offline Q learning method, we take the minimum Q value of the states predicted by the environment model and penalize the predicted reward based on the uncertainty of the environment model to robustly deal with the overestimations.
\end{itemize}
Furthermore, extensive simulated and real-world experiments validate the superiority of the proposed PE-MORL algorithm compared with the state-of-the-art auto-bidding policies.
	Particularly, we validate that the designed environment model exhibits improved generalization ability compared to a conventional neural network, which is not permutation equivariant, and to the GSP-based offline simulator commonly used in SRLB.




\begin{algorithm}[t]
	\caption{MRLB Paradigm}
	\label{alg:mrlb}

	\begin{algorithmic}[1]
        \REQUIRE Real data $\mathcal{D}_\text{R}$
        \ENSURE Auto-bidding policy $\hat{\pi}$
\STATE Initialize the imaginary data $\mathcal{D}_\text{I}\leftarrow\emptyset$ and the auto-bidding policy $\hat{\pi}$;
	\STATE Learn an environment model $\hat{M}$ based on $\mathcal{D}_\text{R}$ with supervised learning;
 \REPEAT
 \STATE Interact with $\hat{M}$ using $\hat{\pi}$ to generate imaginary data and store it in $\mathcal{D}_\text{I}$;
 \STATE Update $\hat{\pi}$ with $\mathcal{D}_\text{R}\cup\mathcal{D}_\text{I}$;
 \UNTIL{The expected return of $\hat{\pi}$ in $\hat{M}$ converges.}
	\end{algorithmic}
\end{algorithm}

\section{Related Works}
We briefly review the existing algorithms based on the MRLB paradigms in the following, and the complete related works, including the mean-field  game algorithms, offline RL algorithms, and permutation equivariance RL algorithms, are given in Appendix \ref{app:related_work}.

\textbf{Model-based RL Bidding (MRLB):}
The design of algorithms based on the MRLB paradigm can date back to \cite{cai2017real}. However, it attempts to model the environment at the granularity of individual impression opportunities, which is impractical in online advertising systems. 
Very recently, \cite{chen2023model} proposed an algorithm based on the MRLB, named Model-Based Automatic Bidding (MBAB). However, MBAB resembles the vanilla MRLB algorithm, which can perform poorly when the real data is fixed as stated in Section \ref{sec:intro}.   
Notably, the proposed PE-MORL can address the OOD challenge caused by the fixed real data and is practical in online advertising systems.

\section{Notations and Preliminaries}
\label{sec:pre}
\textbf{Key Notations:}
$[N]$ denotes the positive integer set containing $1$ to $N$, where $N\in\mathbb{N}_+$;
$[\cdot]_+\triangleq\max\{\cdot, 0\}$; $|\mathcal{D}|$ represents the number of elements in the set $\mathcal{D}$; $\circ$ is the function composition operation; $\left\|\cdot\right\|_1$ denotes the 1-norm;
$\cup$ denotes the union operator between sets, and $\setminus$ denotes the subtraction operator between sets;

In the following, we define some important concepts, operators, and properties that will be used later.  
\begin{definition}[Ordered Vector]
An ordered vector, denoted as $\mathbf{x}\triangleq x_{1:N}\triangleq[x_1,x_2,\cdots,x_N]^\top$, is a vector with elements arranged in a specific sequence, where $x_i$ is the $i$-th element, $i\in[N]$.
\end{definition}

\begin{definition}[Permutation Operator] 
A permutation operator $\rho$ is a function that rearranges the elements of an ordered vector into a different sequence.
\end{definition}
Let $\Omega_N$ be the set of permutation operators operating on $N$-dimensional ordered vectors. Note that there are $N!$ distinct permutation operators in $\Omega_N$.
We next define the permutation equivariance and permutation invariance concepts that will play key roles in our later design of the environment model.
\begin{definition}[Permutation Equivariant and Permutation Invariant]
	A function $f:\mathcal{X}\rightarrow\mathcal{Y}$ is permutation equivariant (PE) if it holds that $f(\rho\mathbf{x})=\rho\mathbf{y}, \forall \rho$, where $\mathbf{x}\in\mathcal{X}$ and $\mathbf{y}=f(\mathbf{x})\in\mathcal{Y}$ are ordered vectors. 
	Moreover, a function $f:\mathcal{X}\rightarrow\mathbb{R}$ is permutation invariant (PI) if it holds that $f(\rho\mathbf{x})=y,\forall \rho$, where $y=f(\mathbf{x})$.
\end{definition}
\subsection{Single-agent Reinforcement Learning}

Single-agent RL \cite{sutton1999policy} focuses on learning optimal sequential decisions for an agent by interacting with an environment. The decision-making processes is typically modeled as a \emph{Partially Observed Markov Decision Process} (POMDP) \cite{krishnamurthy2016partially} $<\mathcal{S}, \mathcal{A}, \mathcal{O}, \Omega,  R, P, \gamma>$. At each time step $t \in \mathbb{N}_+$, $s_t \in \mathcal{S}$ is the state of the agent at time step $t$, $o_t \in \mathcal{O}$ is the observation, $a_t \in \mathcal{A}$ is the operation performed by the agent to interact with the environment. Then $R$ determines the immediate reward $r_t$ to the agent, and the agent's next state $s_{t+1}$ are transitioned by $P$. The agent's task is to find a policy $a_t=\pi(s_t)$ that can maximize the cumulative rewards. Besides, $\gamma\in[0,1]$ is the discounted factor.

The key to most RL methods lies in learning an optimal Q function. The Q function, denoted as $Q^\pi(s_t, a_t)$, represents the expected cumulative reward from taking action $a_t$ in the state $s_t$ and following a certain policy hereinafter. To evaluate the Q function of the current policy, a common approach is to train it with the Mean Squared Error (MSE) loss between the current Q value and the target Q value, i.e., 
\begin{align}
\label{equ:Q_loss}
\mathcal{L} = \mathbb{E}_{(s_t, a_t, r_t, s_{t+1}) \sim D} \left[ ( Q^\pi(s_t, a_t) - \hat{\mathcal{B}}Q^\pi(s_t,a_t) )^2 \right],
\end{align}
where 
\begin{align}   
\label{equ:Q_target}
\hat{\mathcal{B}}Q^\pi(s_t,a_t)=r_t(s_t,a_t)+\gamma \max_{a_{t+1}} Q^\pi(s_{t+1},a_{t+1})
\end{align}
is the target Q value, computed by the Bellman operator.

\begin{figure*}[t]
	\vskip 0.2in
	\begin{center}
		\centerline{\includegraphics[width=160mm]{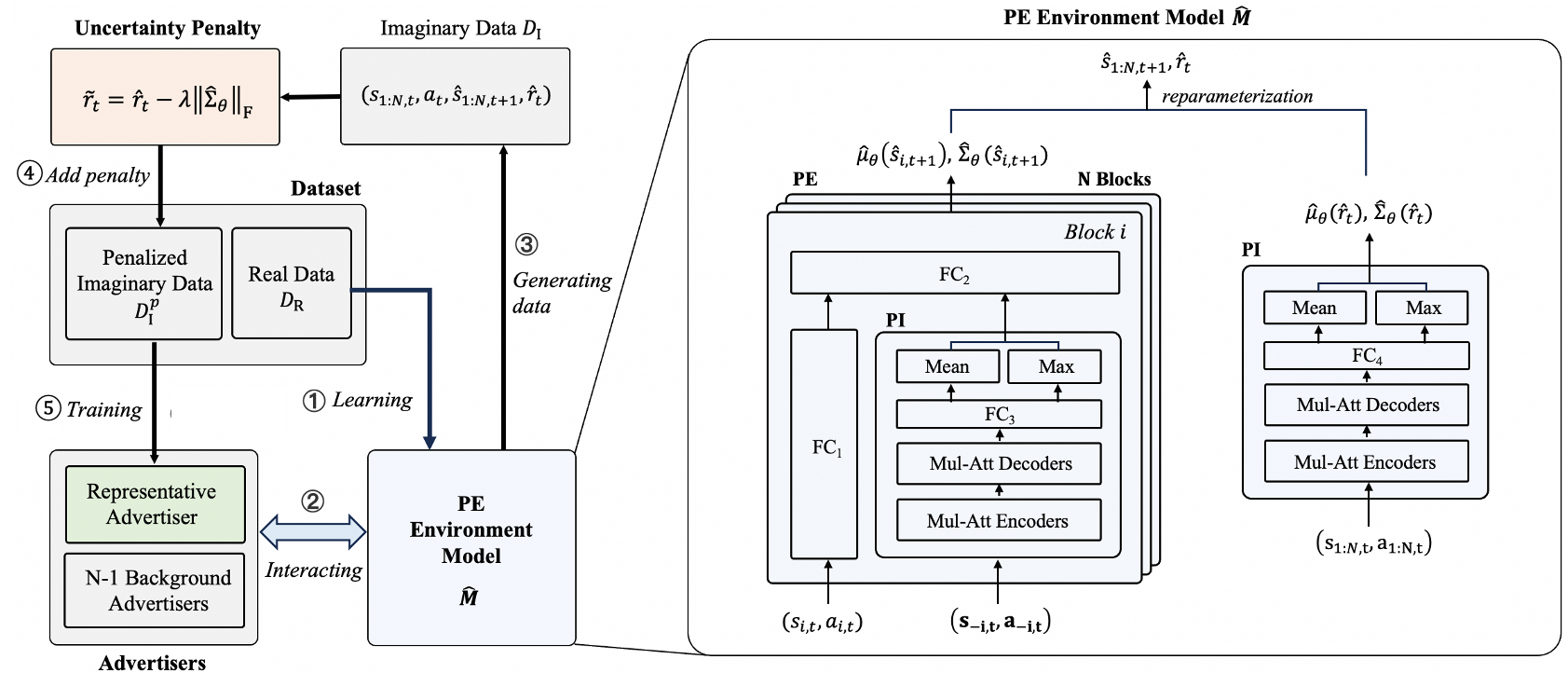}}
		\caption{The proposed PE-MORL algorithm is designed based on the MBRL paradigm.
  The PE-MORL follows the MBRL paradigm: it first learns an environment model and then trains the auto-bidding policy of the representative advertiser with the environment model and the real data. Moreover, the PE-MORL features two novel designs. The first design is the PE environment model with better generalization ability, whose architecture is shown on the right. The second design is the uncertainty penalty to avoid the OOD challenge.}
		\label{fig:algo}
	\end{center}
	\vskip -0.2in
\end{figure*}
\section{Problem Settings}
\label{sec:problem_settings}
In online advertising, the bidding process typically involves $N(>1)$ advertisers competing for impression opportunities simultaneously with a certain policy each, such as the auto-bidding policy.
Here, we focus on learning the optimal auto-bidding policy of a single advertiser, referred to as the \emph{representative advertiser}, under the condition where the policies of other $N-1$ advertisers, referred to as the \emph{background advertisers}, are fixed. 
Specifically, the bidding process of the representative advertiser can be modeled as a POMDP, $<\mathcal{S}, \mathcal{A}, \mathcal{O}, \Omega,  R, P, \gamma>$, with $T\in\mathbb{N}_+$ time steps, where $\mathbf{s}_t\in\mathcal{S}$ denotes the global state of all advertisers at time step $t\in[T]$, and $a_{N,t}\in\mathcal{A}$ denotes the local action, i.e., the bid, of the representative advertiser at time step $t$. 
Due to privacy concerns, the representative advertisers can only obtain an observation $o_{t}\in\mathcal{O}$ which does not contains any information about other advertisers and is also called as local state $s_{N,t}$. In practice, $o_{t}$ contains individual information of advertiser(e.g. the constraint and budget of advertiser, the category and sales history of their product, etc.), the current status of advertiser(e.g. time left, budget left, budget spend speed, etc.) and the prediction of the ad slot valuations, etc. 
Besides, We denote the action of background advertiser $i$ at time step $t$ as $a_{i,t}$, and the observation as $s_{i,t}$. The space of these two variables are also denoted as $\mathcal{A}$ and $\mathcal{O}$.
Then, the global state can be considered as the concatenation of all local states $s_{1:N,t}$. 
And let the joint action of all advertisers at time step $t$
be $\mathbf{a}_t\triangleq a_{1:N,t}\in\mathcal{A}^N$, where $\mathcal{A}^N$ denotes the joint action space.

Hence, the policy \footnote{For convenience, we refer to the auto-bidding policy simply as the ``policy".} of the representative advertiser should be constructed as a mapping from the observation space to the probabilistic space of the action space, i.e., $\pi:\mathcal{O}\rightarrow\Delta(\mathcal{A})$.
Moreover, $R:\mathcal{S}\times\mathcal{A}^N\rightarrow\mathbb{R}$ is the reward function of the representative advertiser, and $P:\mathcal{S}\times\mathcal{A}^N\rightarrow\Delta(\mathcal{S})$ is the transition rule. 
The bidding process can be described as follows: 
at each time step $t$, the representative advertiser takes action $a_{N,t}\sim\pi(o_t)$ and each background advertiser $i$ takes action $a_{i,t}$ at the same time; then the representative advertiser receives the total value of impression opportunities won by her or him
as 
the reward $r_t=R(\mathbf{s}_t, \mathbf{a}_t)$ and pays the corresponding cost $c_t$; afterward, 
all the advertisers transition to the next state $\mathbf{s}_{t+1}$ according to the transition rule $P$.
Besides, $\gamma=1$ is used in our system. 
The goal of the considered auto-bidding policy is to maximize the expected accumulative rewards of the representative agent under the budget constraint, i.e., 
\begin{align}
\label{equ:prob}
\max_{\pi}\;\mathbb{E}_{\pi,P}\bigg[\sum_{t=1}^{T}\gamma^tr_t\bigg],\quad\mathrm{s.t.} \; \sum_{t=1}^{T}c_t\le B,
\end{align}
where $B>0$ denotes the budget of the representative advertiser.
We next propose an MRLB algorithm, named PE-MORL, 
to solve \eqref{equ:prob}.

\section{Algorithm}
Fig. \ref{fig:algo} shows the proposed PE-MORL algorithm.
Compared with the vanilla MBRL paradigm described in Algorithm. \ref{alg:mrlb}, the PE-MORL features two novel designs. Firstly, we design the environment model as a PE function to improve its generalization ability with theoretical guarantees. 
Then, to further reduce the negative impact of the inaccuracies of  the environment model, especially the overestimations, on policy training, we design a robust offline Q learning method that can enhance the reliability of the imaginary data.
In the following, we dive into the details of these two designs and summarize the whole algorithm in Algorithm \ref{alg:PE-MORL}.
\subsection{Permutation Equivariant Environment Model}
\label{sec:pe_environment_model}
As the environment model inherently models the reward function $R$ and the transition rule $P$ --- both are functions of the global state space $\mathcal{S}$ and the joint action space $\mathcal{A}^N$, we design the environment model as a function $\hat{M}$ that outputs the predictions on the next global state $\hat{\mathbf{s}}_{t+1}$ and the reward $\hat{r}_t$ with the global state $\mathbf{s}_t$ and the joint action $\mathbf{a}_t$ as inputs.
Note that although the representative advertiser cannot get $\mathbf{s}_t$ and $\mathbf{a}_t$ during the online inference, we can obtain them during the offline training \footnote{Note that the policy of the representative advertiser $\pi$ can still only take the observation $o_t$ as input during the offline training.} and use them to construct the environment model $\hat{M}$. 
Moreover, we can prove that the reward function $R$ is a PI function and the transition rule $P$ is a PE function in the following proposition.
\begin{proposition}[The PE and PI Properties of the POMDP]
	\label{prop:homo}
In the POMDP of the representation advertiser, the transition rule $P$ is a PE function, i.e., $	
\rho\mathbf{s}_{t+1}\sim P(\rho\mathbf{s}_t,\rho\mathbf{a}_t)$, and 
the reward $R$ is a PI function, i.e., $R(\rho\mathbf{s}_t,\rho\mathbf{a}_{t})=R(\mathbf{s}_t,\mathbf{a}_{t})$, $\forall t, \rho$.
\end{proposition} 
In essence, the reason lies in the PE and PI property of the auction mechanism in the online advertising system concerning the local states and the bids of all advertisers. The detailed proof is given in Appendix \ref{app:proof_prop_homo}.
Hence, we design the environment model $\hat{M}$ to be a PE function from the perspective of $\hat{\mathbf{s}}_{t+1}$ and a PI function from the perspective of $\hat{r}_t$, i.e., for any permutation operator $\rho$, we have:
\begin{align}
\label{equ:design_of_model}
 \rho\hat{\mathbf{s}}_{t+1},\hat{r}_t = \hat{M}(\rho\mathbf{s}_t,\rho\mathbf{a}_t).
\end{align}
We name the designed environment model $\hat{M}$ as the \emph{PE environment model} \footnote{Note that as \eqref{equ:design_of_model} can be equivalently expressed as \eqref{equ:equivariant_design_sim}, we directly name the environment model as the PE environment model.}.
A direct insight behind this design is that the information contained in $\mathbf{s}_t$ and $\mathbf{a}_t$ does not change no matter how we swap the elements in them, and hence, $\hat{r}_t$ and the values of the elements in $\hat{s}_{t+1}$ should not be changed, only the positions of the elements in $\hat{s}_{t+1}$ should be swapped accordingly \cite{li2021permutation}. 
Moreover, as we know one of the most desirable characteristics of the environment models is their strong ability for generalization. 

We evaluate the generalization ability of an environment model $M$ by the gap between its accuracy on the training dataset $\mathcal{S}$, and its accuracy on the test dataset $\mathcal{D}$, i.e., 
\begin{align}
	\Delta_G(\hat{M})\triangleq\mathbb{E}_{\mathcal{D}}[L(\hat{M}, \mathbf{s_t}, \mathbf{a}_t)]-\frac{1}{|\mathcal{S}|}\sum_{\mathcal{S}}L(\hat{M}, \mathbf{s_t}, \mathbf{a}_t),
\end{align}
where $L(\hat{M}, \mathbf{s_t}, \mathbf{a}_t)$ is the loss function (here we use the mean absolute error (MAE) as the loss function) and can represent the accuracy of the environment model.
Note that a smaller gap $\Delta_G(\hat{M})$ indicates a better generalization ability of the environment model $M$ since we can reduce the loss of the environment model on the training dataset to obtain a lower value of the loss of the environment model on the test dataset, which indicates a better generalization ability.
Notably, we can prove that the design in \eqref{equ:design_of_model} can reduce the upper bound of the gap $\Delta_G(\hat{M})$.

\begin{theorem}[Better Generalization Ability of the PE Environment Model]
	For any environment model $\hat{M}'$, making it satisfy  \eqref{equ:design_of_model} can reduce the upper bound of the gap $\Delta_G(\hat{M}')$, i.e.,
	\begin{align}
        \label{equ:upper_bound_pe}
		 \text{Upper Bound of }  \Delta_G(\hat{M})\le\text{Upper Bound of }  \Delta_G(\hat{M}'),
	\end{align}
	where $\hat{M}$ denotes the environment model that satisfies  \eqref{equ:design_of_model}.
\end{theorem}
The proof is given in Appendix \ref{app:proof_thm}. Therefore, the environment model with the design of \eqref{equ:design_of_model} can have a better generalization ability.


\subsubsection{Architecture Design:}

To design an environment model $\hat{M}$ that satisfies \eqref{equ:design_of_model}, a straightforward way is to randomly build a function approximator, e.g. a neural network, and apply the orbit averaging $\mathcal{Q}$ to it, as is described previously. However, as the number of background advertisers is usually large and applying the orbit averaging requires $N!$ times of calculations every time, it is time infeasible in practice.
Hence, we here design a neural network as the environment model $\hat{M}$ with the structure that naturally satisfies \eqref{equ:design_of_model}. We next dive into the details of its structure.
\begin{remark}
    Intuitively, a neural network that naturally satisfies \eqref{equ:design_of_model} can be expected to easily approximate the functions after orbit averaging $\mathcal{Q}$ than a vanilla neural network does, especially when the training dataset is not sufficient \cite{qin2022benefits, duan2022context}. 
\end{remark}
As shown in Fig. \ref{fig:algo}, the environment model $\hat{M}$ is mainly composed of two neural networks, where the left one is a PE function, responsible for predicting the next state $\hat{\mathbf{s}}_t$ and the right one is a PI function, responsible for predicting the reward $\hat{r}_t$.

\textbf{The Left Neural Network}
contains $N$ \emph{blocks} with identical structures and shared parameters, where block $i$ 
takes the global state $\mathbf{s}_t$ and the joint action $\mathbf{a}_t$ as inputs and generates the prediction on the local state, $\hat{s}_{i,t+1}$.
Specifically, 
block $i$ is composed of three modules, including (1) a fully connected neural network, named FC$_1$, used to extract the embedding of the corresponding advertiser; (2) a PI neural network used to extract the embedding of all other advertisers; and (3) a fully connected neural network, named FC$_2$, used to deal with the embeddings generated by the first two modules and output the predictions $\hat{s}_{i,t+1}$.
The insight of designing a PI neural network for feature extractions of all other agents is that: the input order of other agents can neither change the information contained in $(\mathbf{s}_{-i,t},\mathbf{a}_{-i,t})$ nor influence the situation the corresponding advertiser faces. 
Notably, the structure of the PI neural network is designed based on the following lemma. 
\begin{lemma}[Theorem 2 in \cite{zaheer2017deep}]
\label{lemma:pi}
All PI functions can be represented as $\phi(\text{pool}(\psi(\cdot)))$ when $\text{pool}(\cdot)$ is the sum operator and $\phi(\cdot)$ and $\psi(\cdot)$ are continuous. 
\end{lemma}
The PI neural network is designed in the form of $\phi(\text{pool}(\psi(\cdot)))$. Specifically, we use several multi-head attention (Mul-Att) encoders and decoders \cite{vaswani2017attention} as $\psi(\cdot)$. Note that there is a fully connected neural network at the end of $\psi(\cdot)$ across attention heads to extract features between different heads. 
In block $i$, we denote the output of $\psi(\cdot)$ as $\mathbf{B}_{i,t}\triangleq[\mathbf{b}^i_{0,t},\cdots,\mathbf{b}^i_{i-1,t},\mathbf{b}^i_{i+1,t},\cdots,\mathbf{b}^i_{N,t}]$, where $\mathbf{b}^i_{j,t}\in\mathbb{R}^B$ is the extracted feature for the advertiser with $(s_{j,t},a_{j,t})$, where $j\in[N]\cup\{0\}\setminus\{i\}$.
For the $pool(\cdot)$ function, we use two functions, including a mean function that is inherently 
the same as the sum operator (as required in Lemma \ref{lemma:pi}), and a max function.  
The insight for using the max function is that: for each advertiser, usually only the highest bid of other advertisers will influence its reward and state transition.
Let $\mathbf{m}_t^i\triangleq m^i_{1:N,t}\in\mathbb{R}^B$ and $\mathbf{g}_t^i\triangleq g^i_{1:N,t}\in\mathbb{R}^B$ be the outputs of the mean function and the max function, respectively.  
Then we have $\forall k\in[B]$:
\begin{align}
m_{k,t}^i=\frac{1}{N}\sum_{j\in[N]\cup\{0\}\setminus\{i\}}b^{i,k}_{j,t},
\end{align}
where $b^{i,k}_{j,t}$ denotes the $k$-th element in $\mathbf{b}^i_{j,t}$, and 
\begin{align}
    {g}_{k,t}^i=\max_j\bigg\{b^{i,k}_{j,t}\bigg|j\in[N]\cup\{0\}\setminus\{i\}\bigg\}.
\end{align}
Afterward, $\mathbf{m}^i_t$ and $\mathbf{g}^i_t$ are concatenated and processed by FC$_2$ together with the outputs of FC$_1$. Here FC$_2$ can be viewed as the function $\phi(\cdot)$ in Lemma \ref{lemma:pi}.
Note that we assume that the ground truth ${s}_{i,t+1}$ obeys a certain Gaussian distribution. 
Hence, FC$_2$ is designed to generate the mean value of $\hat{s}_{i,t+1}$, denoted as $\hat{\mu}_\theta(\hat{s}_{i,t+1})$, and its variance and covariances with the predictions on all other next local states and the reward, denoted as $\hat{\Sigma}_\theta(\hat{s}_{i,t+1})$, where $\theta$ represents the trainable parameters of all the environment model $\hat{M}$.
As different blocks have the same structure and share the parameters, all the $N$ blocks constitute a PE function. 

\textbf{The Right Neural Network} contains only a single block with the same architecture as the PI neural network on the left but with independent parameters and different input dimensions. 
Specifically, the Mul-Att encoder takes $(\mathbf{s}_t,\mathbf{a}_t)$ as inputs rather than $(\mathbf{s}_{-i,t},\mathbf{a}_{-i,t})$.
The right neural network generates the predictions on the mean of $\hat{r}_t$, denoted as $\hat{\mu}_\theta(\hat{r}_t)$ as well as its variance and covariances with the predictions on the next local states, denoted as $\hat{\Sigma}_\theta(\hat{r}_t)$. 

\textbf{The Outputs} of the left and right neural networks constitute the predictions on the mean, denoted as $\hat{\mathbf{\mu}}_\theta$, and the covariance matrix, denoted as $\hat{\Sigma}_\theta$,
of all the next local states and the reward. Specifically, the mean $\hat{\mathbf{\mu}}_\theta$ can be represented as:
\begin{align}
\hat{\mu}_\theta\triangleq
[\hat{\mu}_\theta(\hat{s}_{1,t+1}),\cdots,\hat{\mu}_\theta(\hat{s}_{N,t+1}), \hat{\mu}_\theta(\hat{r}_t)].
\end{align}
To guarantee the symmetric property of the covariance matrix, the covariance matrix $\hat{\Sigma}_\theta$ is designed as:
\begin{align}
     \hat{\Sigma}_\theta=\frac{1}{2}(\tilde{\Sigma}_\theta+\tilde{\Sigma}^\top_\theta),
\end{align}
where $\tilde{\Sigma}_\theta$ is defined as:
\begin{align}
    \tilde{\Sigma}_\theta\triangleq[\hat{\Sigma}_\theta(\hat{s}_{1,t+1}),\cdots,\hat{\Sigma}_\theta(\hat{s}_{N,t+1}), \hat{\Sigma}_\theta(\hat{r}_{t})]^\top.
\end{align}
Then, the predicted next local state $\hat{s}_{i,t+1}$ and the reward $\hat{r}_t$ are generated by the re-parameterization method \cite{mohamed2020monte} based on $\hat{\mu}_\theta$ and $\hat{\Sigma}_\theta$, i.e.,
\begin{align}
    \begin{bmatrix}
        \hat{\mathbf{s}}_{t+1}\\
        \hat{r}_t
    \end{bmatrix}=\hat{\mu}_\theta+\mathbf{L}\mathbf{\epsilon},
\end{align}
where $\mathbf{L}$ is the Cholesky factorization of the covariance matrix, i.e., $\hat{\Sigma}_\theta=\mathbf{L}\mathbf{L}^\top$, and $\epsilon\sim\mathcal{N}(\mathbf{0},\mathbf{I})$ is a random vector obeying the standard Gaussian distribution.


\textbf{The Loss Function} is designed as the maximum log-likelihood function with regularization terms \cite{kendall2017uncertainties}, i.e.,
\begin{align}
\label{equ:loss_function}
    \mathcal{L}(\theta)=\mathbb{E}_{[\mathbf{s}_{t+1}^\top,r_t]^\top\sim\mathcal{D}_\text{R}}\bigg[&\big(\begin{bmatrix}
        {\mathbf{s}}_{t+1}\\
        {r}_t
    \end{bmatrix}-\hat{\mu}_\theta\big)^\top{\Sigma}_\theta^{-1}\big(\begin{bmatrix}
        {\mathbf{s}}_{t+1}\\
        {r}_t
    \end{bmatrix}-\hat{\mu}_\theta\big)\notag\\
    &+\log|\hat{\Sigma}_\theta|+\|\hat{\Sigma}_\theta\|_F
    \bigg],
\end{align}
where $\mathcal{D}_\text{R}=\{(\mathbf{s}_t,\mathbf{a}_t, r_t,\mathbf{s}_{t+1})\}$ denotes the real data collected from the online advertising system, containing the real transition pairs. 
To train the equivariant environment model $\hat{M}$, we apply the gradient descent method based on $\mathcal{L}(\theta)$, i.e.,
\begin{align}
\label{equ:update_env}
    \theta\leftarrow\theta - \alpha\nabla_\theta\mathcal{L}(\theta),
\end{align}
where $\alpha>0$ is the update step.

\textbf{Ensemble Environment Models:} To increase the robustness of data generated by the environment model, we train an ensemble of $K$ environment models. During the training phase, the auto-bidding policy interacts with all $K$ environment models, and we randomly pick the prediction of an environment model as the final prediction and store it into the imaginary data, denoted as $\mathcal{D}_\text{I}$. Specifically, the imaginary data consists of the state transition pairs of the representative advertiser, i.e., 
$\mathcal{D}_\text{I}=\{(o_t,a_{N,t}, \hat{r}_t,\hat{o}_{t+1})\}$, where $\hat{o}_{t+1}=O(\hat{\mathbf{s}}_{t+1})$.

\subsection{Robust Offline Q Learning Method}
Given that the real data $\mathcal{D}_\text{R}$ is fixed, 
the environment model is not subject to iterative enhancement concurrent with the policy training. 
Hence, the inaccuracies of the environment model
at certain states outside the real data $\mathcal{D}_\text{R}$ will persist and may seriously mislead the training of the policy. 
Particularly, the environment model's overestimations of rewards 
can seriously jeopardize the policy since it might incorrectly deviate from the original data $\mathcal{D}_\text{R}$, leading to a significant performance decline. In addition, if the states predicted by the environment model have erroneous large Q values, then it can also mislead the policy since they act as the target Q value in the Q loss function as described in \eqref{equ:Q_loss}.

Our main idea to address these challenges is to pessimistically treat the predicted rewards of the environment model and the Q values of its predicted states. Specifically, in the robust offline Q learning method, we modify the target Q value $\hat{\mathcal{B}}Q(s_t,a_t)$ in the loss function of Q as:
\begin{align}
\label{equ:target_q}
\hat{\mathcal{B}}Q(s_t,a_t)=\underbrace{\hat{r}_t-\lambda\|\hat{\sigma}_t\|_F}_{\text{ reward penalty}} + \gamma \underbrace{\min_{\hat{s}_{t+1}\in \hat{M}(s, a)}(\max_{a'\in\mathcal{A}}Q(\hat{s}_{t+1}, a'))}_{\text{state penalty}},
\end{align}
where $\hat{M}(s, a)$ denotes the set of the samples of the environment model's prediction on states, $\hat{\sigma}_t$ denotes the variance of the ensemble environment model's predicted rewards, and $\lambda>0$ is a constant.
In the state penalty, we take the minimum of the max Q value of the predicted state to treat the predictions of the environment model pessimistically. This can increase the robustness of the training policy even in the worst case.  As for the reward penalty, we note that the environment model's inaccuracy is higher in states outside the real data $\mathcal{D}_\text{R}$ and lower in states within it, which is consistent with the value of $\hat{\sigma}_t$ \cite{yu2020mopo,schulz2018tutorial}. Therefore, 
penalizing it can make the policy more pessimistic about states outside the real data $\mathcal{D}_\text{R}$, thereby decreasing the risk of deviating incorrectly from the real data $\mathcal{D}_\text{R}$. 
Intuitively, the penalized predicted reward $\tilde{r}\triangleq\hat{r}_t-\lambda\|\hat{\sigma}_t\|_F$ could have lower values as the state-action pairs move away from the real data $\mathcal{D}_\text{R}$, which can constrain the trajectories of the policy within the proximity of $\mathcal{D}_\text{R}$.
Theoretically, we can prove that optimizing the policy in the environment model $\hat{M}$ with penalized predicted reward $\tilde{r}_t$ can improve the lower bound of the policy's performance in the online advertising system $M$.
\begin{proposition}[Lower Bound Improvement]
\label{prop:lower_bound}
    For any policy $\pi$, its performance in the online advertising system $M$, denoted as $\eta_M(\pi)$, is lower bounded by its performance in a fitted environment model $\hat{M}$ with the penalized predicted reward $\tilde{r}_t$, denoted as $\eta_{\hat{M}}(\pi)$, i.e., $\eta_M(\pi)\ge\eta_{\hat{M}}(\pi)$.
\end{proposition}
The proof is given in Appendix \ref{app:proof_lower_bound}, where we leverage the proof technique in \cite{yu2020mopo}.

\textbf{Penalized Imaginary Data:}
After adding the penalty to the predicted reward $\hat{r}_t$ in the imaginary data, we further obtain the penalized imaginary data $\mathcal{D}_\text{I}^p=\{(\mathbf{s}_t,\mathbf{a}_t,\tilde{r}_t,\hat{\mathbf{s}}_{t+1})\}$.
The penalized imaginary data $\mathcal{D}_\text{I}^p$ and the real data $\mathcal{D}_\text{R}$ together constitute the training dataset used for policy training. 

\textbf{Policy Training:} We leverage an actor-critic framework in training the policy, where the Q function and the policy are updated iteratively \cite{haarnoja2018soft}.
The loss function of the policy is its negative return, which is exactly $-Q^{\pi_w}$, i.e., 
\begin{align}
\label{equ:learn_policy}
l(w)=\mathbb{E}_{\mathcal{D}_\text{R}\cup\mathcal{D}^p_\text{I}}\bigg[-Q^{\pi_w}(o_t,a_{N,t})\bigg].
\end{align}
where $w$ denotes the parameter of the policy.
We use the gradient descent to update $w$, i.e.,
\begin{align}
\label{equ:policy_update}
    w\leftarrow w-\eta\nabla_w l(w),
\end{align}
where $\eta>0$ is the step size.

\subsection{Overall Algorithm}
The overall algorithm is described in Algorithm \ref{alg:PE-MORL}. Specifically, we first learn a PE environment model $\hat{M}$ based on the real data $\mathcal{D}_\text{R}$.
Then we train the Q function with the target \eqref{equ:target_q} and learn the
policy of the representation advertiser with \eqref{equ:learn_policy} based on the imaginary data and the real data $\mathcal{D}_\text{R}\cup\mathcal{D}^p_\text{I}$.
Specifically, at each iteration,
we interact with the trained equivariant environment model $\hat{M}$ using the policy $\pi_w$ \footnote{Note that we assume the policies of all the background advertisers are known in advance.}.
The imaginary data $(\mathbf{s}_t,\mathbf{a}_t, \hat{r}_t,\hat{\mathbf{s}}_{t+1})$ will be generated during the interactions. After penalizing the reward, we store $(\mathbf{s}_t,\mathbf{a}_t, \tilde{r}_t,\hat{\mathbf{s}}_{t+1})$ into the penalized imaginary data $\mathcal{D}^p_\text{I}$.
With $\mathcal{D}_\text{R}\cup\mathcal{D}^p_\text{I}$, we update the policy parameter $w$ based on \eqref{equ:policy_update}.
Update the policy parameter $w$ until the return of ${\pi}_w$ in $\hat{M}$ converges. 
Then we output ${\pi}_w$ as an approximate solution to \eqref{equ:prob}.

\begin{algorithm}[tb]
	\caption{PE-MORL}
	\label{alg:PE-MORL}
	\begin{algorithmic}[1]
 \REQUIRE The real data $\mathcal{D}_\text{R}$, hyper-parameters $\lambda,\alpha,\beta,\eta$;\\
	\ENSURE ${\pi}_w$ (that can be viewed as the solution to \eqref{equ:prob}), ;\\
 \STATE Initialize the penalized imaginary data $\mathcal{D}^p_\text{I}\leftarrow\emptyset$ and the policy parameter $w$;
	\STATE Learn an equivariant environment model $\hat{M}$ based on the real data $\mathcal{D}_\text{R}$ with \eqref{equ:update_env};
 \REPEAT
 \STATE Interact with $\hat{M}$ using ${\pi}_w$ to generate the  imaginary data $(\mathbf{s}_t,\mathbf{a}_t, \hat{r}_t,\hat{\mathbf{s}}_{t+1})$;
 \STATE Obtain the penalized predicted reward $\tilde{r}_t$ 
 and store $(\mathbf{s}_t,\mathbf{a}_t, \tilde{r}_t,\hat{\mathbf{s}}_{t+1})$ in $\mathcal{D}_\text{I}^p$;
 \STATE Update the Q function $Q^{\pi_w}$ with loss function \eqref{equ:Q_loss}, where the target Q value leverages \eqref{equ:target_q}.
 \STATE Update ${\pi}_w$ with $\mathcal{D}_\text{R}\cup\mathcal{D}^p_\text{I}$ based on \eqref{equ:policy_update};
 \UNTIL{The expected return of ${\pi}_w$ in $\hat{M}$ converges.}
	\end{algorithmic}
\end{algorithm}
\begin{table*}[t]
	\caption{A/B Testing in the real-world experiments: comparing V-CQL, USCB, and MBAB with PE-MORL: $3k$ representative advertisers lasting for $10$ days.}
 \vspace{-6mm}
 \newcommand\xrowht[2][0]{\addstackgap[.5\dimexpr#2\relax]{\vphantom{#1}}}
	\label{table:online_exp}
	\vskip 0.15in
	\begin{center}
		\begin{small}
			\begin{tabular}{cccccccccccc}
				\toprule
				Algorithms&GMV& ROI & Cost &Algorithms&GMV& ROI & Cost & Algorithms&GMV& ROI & Cost \\
				\toprule
    V-CQL &6,313,011&4.50&1,402,298 &USCB &6,059,062&5.09&1,190,859&
    MBAB &7,200,474&4.57&1,574,216
    \\
    \midrule
    \makecell[c]{\textbf{PE-MORL} \\\textbf{(Ours)}}
    &\textbf{6,591,894}&\textbf{4.68}&\textbf{1,409,574}
    & \makecell[c]{\textbf{PE-MORL} \\\textbf{(Ours)}}
    &\textbf{6,497,634}&\textbf{5.41}&\textbf{1,201,845}
    & \makecell[c]{\textbf{PE-MORL} \\\textbf{(Ours)}}
    &\textbf{7,482,550}&\textbf{4.83}&\textbf{1,548,506}\\
    \midrule
    Diff& \textbf{+4.4\%}&\textbf{+3.9\%}&\textbf{+0.5\%}
    &Diff& \textbf{+7.2\%}&\textbf{+6.3\%}&\textbf{+0.9\%}
    &Diff& \textbf{+3.9\%}&\textbf{+5.7\%}&{-1.6\%}\\
\bottomrule

			\end{tabular}
		\end{small}
	\end{center}
	\vskip -0.1in
\end{table*}
\section{Experiments}
\label{sec:exp}

We conduct extensive real-world experiments to validate the effectiveness of our approach. We mainly study the following four questions in the experiments: (1) What is the overall performance of the PE-MORL algorithm compared to the state-of-the-art auto-bidding algorithms? (2) What is the performance of the equivariant environment model compared to the commonly used bidding environment? Does the equivariance design help with its generalization ability?
(3) Can the proposed model-based approach, PE-MORL, truly enable policies to transcend the limitations imposed by the real data $\mathcal{D}_\text{R}$ compared with the ORLB algorithm?
(4) Does the uncertainty penalty work?

\textbf{Experiment Setup.} We conduct real-world experiments on the online advertising system of one of the world's largest E-commerce platforms, TaoBao. 
We leverage A/B Testing to evaluate our approach on thousands of representative advertisers each with thousands of background advertisers with fixed local policies.
An episode is set to be one day. There are $T=48$ time steps, and the duration between two-time steps is half an hour.

\textbf{Performance Metric.} The main performance metric in our experiments is the objective function, which represents the total value of impression opportunities won by the advertiser in an episode.
This metric is subsequently referred to as the `GMV' throughout the following sections.
Besides, we use other two indexes that are commonly used in the auto-bidding problem to evaluate the performance of the policy. The first metric is the total consumed budget (Cost) of the advertiser. The second metric is the return on investment (ROI) which is defined as the ratio between the GMV and the Cost of the advertiser. Note that larger values of GMV, ROI, and no significant decrease in Cost indicate better performance of the auto-bidding policy.
As for the environment models, we utilize the mean square error (MSE) and the mean absolute error (MAE) to evaluate their performance.

\textbf{Baselines.} We compare the proposed PE-MORL with the state-of-the-art auto-bidding algorithms, including the algorithm based on \textbf{the SRLB: USCB} \cite{he2021unified} that trains policies in a commonly used GSP-based offline simulator, the algorithm based on \textbf{the ORLB:
V-CQL} \footnote{V-CQL is the SORL proposed by \cite{mou2022sustainable} without online explorations. Note that we focus on the offline algorithms and the online exploration is out of scope.} \cite{mou2022sustainable}, a model-free offline RL algorithm; and a \textbf{vanilla MRLB algorithm: MBAB \cite{chen2023model}}.
In the ablation study, 
to validate the effectiveness of the proposed equivariant environment model $\hat{M}$, we compare it with the GSP-based offline simulator and a neural network-based environment model without the design of\eqref{equ:design_of_model}, constructed by only fully connected layers.

\subsection{Real-World Experiments}

\textbf{To Answer Question (1):} 
We conduct A/B Testing on the real-world online advertising system between the proposed PE-MORL and the state-of-the-art baselines for thousands of representative advertisers over multiple days. The results are shown in Table \ref{table:online_exp}.
We can see that the GMV of the PE-MORL significantly exceeds those in the V-CQL, USCB, and MBAB. 
Besides, nearly all other performance metrics, including ROI and Cost of the PE-MORL exceed those in the baselines. 
These validate the superiority of the proposed PE-MORL.

\textbf{To Answer Question (2):} 
We first compare the proposed equivariant environment model $\hat{M}$ with a neural network-based model without the design of \eqref{equ:design_of_model}.
Both models are trained on real data collected from the online advertising system in $7$ days, and evaluated on another real data. The results are given in Table \ref{table:compare_env_model}. We can see that the test loss of $\hat{M}$ is lower than that of the environment model without the design of \eqref{equ:design_of_model}.
Moreover, we compare $\hat{M}$ with the GSP-based offline simulator that is commonly used in the SRLB, and the results are given in Table. \ref{table:compare_env_model_gsp}.
We can see that the equivariant environment model $\hat{M}$ can significantly improve the prediction accuracy on both the next state and reward compared to the GSP-based offline simulator.
\begin{table}[t]
	\caption{The test MAE and MSE losses of the environment model constructed by the fully connected neural network without the design of \eqref{equ:design_of_model} and the designed PE environment model.}
 \vspace{-5mm}
	\label{table:compare_env_model}
	\vskip 0.15in
	\begin{center}
		\begin{small}
			\begin{tabular}{cccc}
				\toprule
				Test Loss&\makecell[c]{Fully-connected \\ Environment \\Model  without \eqref{equ:design_of_model}}& \makecell[c]{PE \\ Environment  \\Model $\hat{M}$}  & Diff \\
				\toprule
				MAE&0.404&\textbf{0.381}&\textbf{-5.9\%}
				\\
    \midrule
    MSE &1.469&\textbf{1.420}&\textbf{-3.3\%}\\
				\bottomrule
			\end{tabular}
		\end{small}
	\end{center}
	\vskip -0.1in
\end{table}

\begin{table}[t]
	\caption{The test MAE and MSE losses of the GSP-based offline simulator and the designed PE environment model.}
  \vspace{-5mm}
	\label{table:compare_env_model_gsp}
	\vskip 0.15in
	\begin{center}
		\begin{small}
			\begin{tabular}{cccc}
				\toprule
				Test Loss&\makecell[c]{GSP-based \\Offline \\
    Simulator}& \makecell[c]{PE \\ Environment  \\Model $\hat{M}$}  & Diff \\
				\toprule
				MAE&1.253&\textbf{0.381}&\textbf{-69.6\%}
				\\
    \midrule
    MSE &12.034&\textbf{1.420}&\textbf{-88.2\%}\\
				\bottomrule
			\end{tabular}
		\end{small}
	\end{center}
	\vskip -0.1in
\end{table}

\begin{figure}[t]
	\vskip 0.2in
	\begin{center}
		\centerline{\includegraphics[width=70mm]{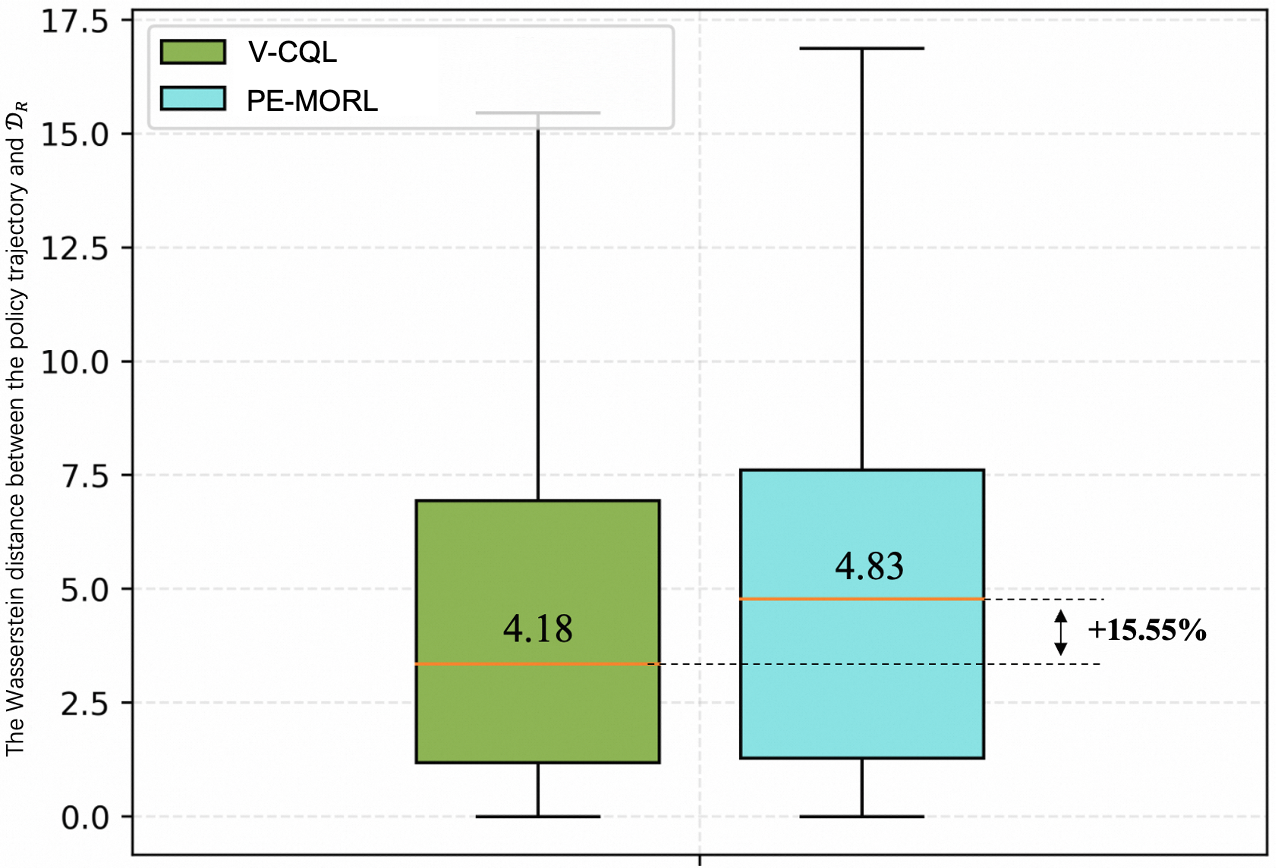}}
		\caption{The Wasserstein distance between the policy trajectories and the real data $\mathcal{D}_\text{R}$.}
		\label{fig:distance}
	\end{center}
	
\end{figure}

\textbf{To Answer Question (3):} 
We compare the distance between the real data $\mathcal{D}_\text{R}$ and the policy trajectories trained by the ORLB algorithm, V-CQL, denoted as $d_{V-CQL}$, and the PE-MORL algorithm, denoted as $d_\text{PE-MORL}$, respectively. The distance is calculated by the Wasserstein metric, and the results are given in Fig. \ref{fig:distance}.
We can see that the average value of $d_\text{PE-MORL}$ is $15.55\%$ higher than $d_\text{V-CQL}$, which indicates that the PE-MORL can indeed break through the real data $\mathcal{D}_\text{R}$ to find better policies. 
\textbf{Visualizations} of the policy trajectories are given in Fig. \ref{fig:distance_show_case}, and the corresponding performances of each single representation advertiser are given in Table. \ref{table:online_exp_show_case}. We can see that the trajectories of policies trained by the PE-MORL are further from the real data and exhibit better performance.
\subsection{Ablation Study}
\textbf{To Answer Question (4):} 
To validate the effectiveness of the uncertainty penalty design, we evaluate the policy under the different values of $\lambda$, and the results are given in Table. \ref{table:ablation_study}. Note that $R/R^*$ and online rate are commonly used offline metrics in the auto-bidding \cite{he2021unified}. Lower $R/R^*$ and online rate indicate the poorer performance of the policy. 
We can see that a smaller value of $\lambda$ can lead to poor policy performance. When $\lambda=0$, there is no uncertainty penalty in the PE-MORL, and the trained policy performs the worst.
This validates the effectiveness of the uncertainty penalty method.

\begin{table}[t]
	\caption{A/B Testing in the real-world experiments between V-CQL and PE-MORL: $2$ representative advertisers lasting for $5$ days.}
\vskip -0.2in
 \newcommand\xrowht[2][0]{\addstackgap[.5\dimexpr#2\relax]{\vphantom{#1}}}
	\label{table:online_exp_show_case}
	\vskip 0.15in
	\begin{center}
		\begin{small}
			\begin{tabular}{cccccc}
				\toprule
    \makecell[c]{Representation \\ Advertiser}
			&	Algorithms&GMV& ROI & Cost \\
				\toprule\xrowht{10pt}
   \multirow{3}{*}{Left in Fig. \ref{fig:distance_show_case}}& V-CQL &16,940&2.25&7,543\\
    \cline{2-5}\xrowht{10pt}
   & \makecell[c]{\textbf{PE-MORL} \textbf{(Ours)}}
    &\textbf{23,995}&\textbf{3.26}&\textbf{7,365}\\
    \cline{2-5}\xrowht{10pt}
   & Diff& \textbf{+41.7\%}&\textbf{+45.1\%}&-2.4\%\\
\bottomrule\xrowht{10pt}
 \multirow{3}{*}{Right in Fig. \ref{fig:distance_show_case}}& V-CQL &241&1.37&175\\
    \cline{2-5}\xrowht{10pt}
   & \makecell[c]{\textbf{PE-MORL} \textbf{(Ours)}}
    &\textbf{502}&\textbf{2.70}&\textbf{186}\\
    \cline{2-5}\xrowht{10pt}
   & Diff& \textbf{+108.5\%}&\textbf{+96.8\%}&\textbf{+6.0\%}\\
\bottomrule
			\end{tabular}
		\end{small}
	\end{center}
	\vskip -0.1in
\end{table}

\begin{table}[t]
	\caption{Ablation study under different $\lambda$.}
 \vskip -0.2in
	\label{table:ablation_study}
	\vskip 0.15in
	\begin{center}
		\begin{small}
			\begin{tabular}{cccccc}
				\toprule
				$\lambda$& $R/R^*$&  Online Rate &$\lambda$& $R/R^*$&  Online Rate \\
				\toprule
				5&\textbf{0.8786}&\textbf{0.834} &  2 & 0.7714&0.783
				\\
    \midrule
    4&\textbf{0.8788}&\textbf{0.822}& 1 & 0.6493&0.401\\
    \midrule
    3 & \textbf{0.8831}&\textbf{0.835} &0&0.4519&0.412\\
				\bottomrule
			\end{tabular}
		\end{small}
	\end{center}
	\vskip -0.1in
 
\end{table}

\section{Conclusions}
\label{sec:conclusion}
In this paper, we analyze a recent paradigm shift in the RL-based auto-bidding technique: from the SRLB to the ORLB and point out the problems existing in both paradigms. 
We then introduce the MRLB, a promising paradigm to solve the problems in both the SRLB and the ORLB. However, the vanilla MRLB algorithm may still have poor performance due to the OOD challenge caused by the fixed real data. Hence, we propose a specific algorithm based on the MRLB, named the PE-MORL, that can further address the OOD challenge.
Compared with the vanilla algorithm derived by the MRLB, the PE-MORL features two novel designs. Firstly, we design the environment model as a PE neural network with theoretical guarantees to increase its generalization ability. Secondly, we design a penalty method on the predicted reward of the environment model. The penalty method can alleviate the negative impact on policy training caused by overestimating the predicted rewards.   
Extensive real-world experiments validate the superiority of the proposed PE-MORL algorithm compared with the state-of-the-art auto-bidding policies.
Particularly, we also validate that the designed environment model has better generalization ability compared to a vanilla neural network that is not permutation equivariant and the GSP-based offline simulator commonly used in the SRLB.

\begin{figure}[t]
	\vskip 0.2in
	\begin{center}
		\centerline{\includegraphics[width=83mm]{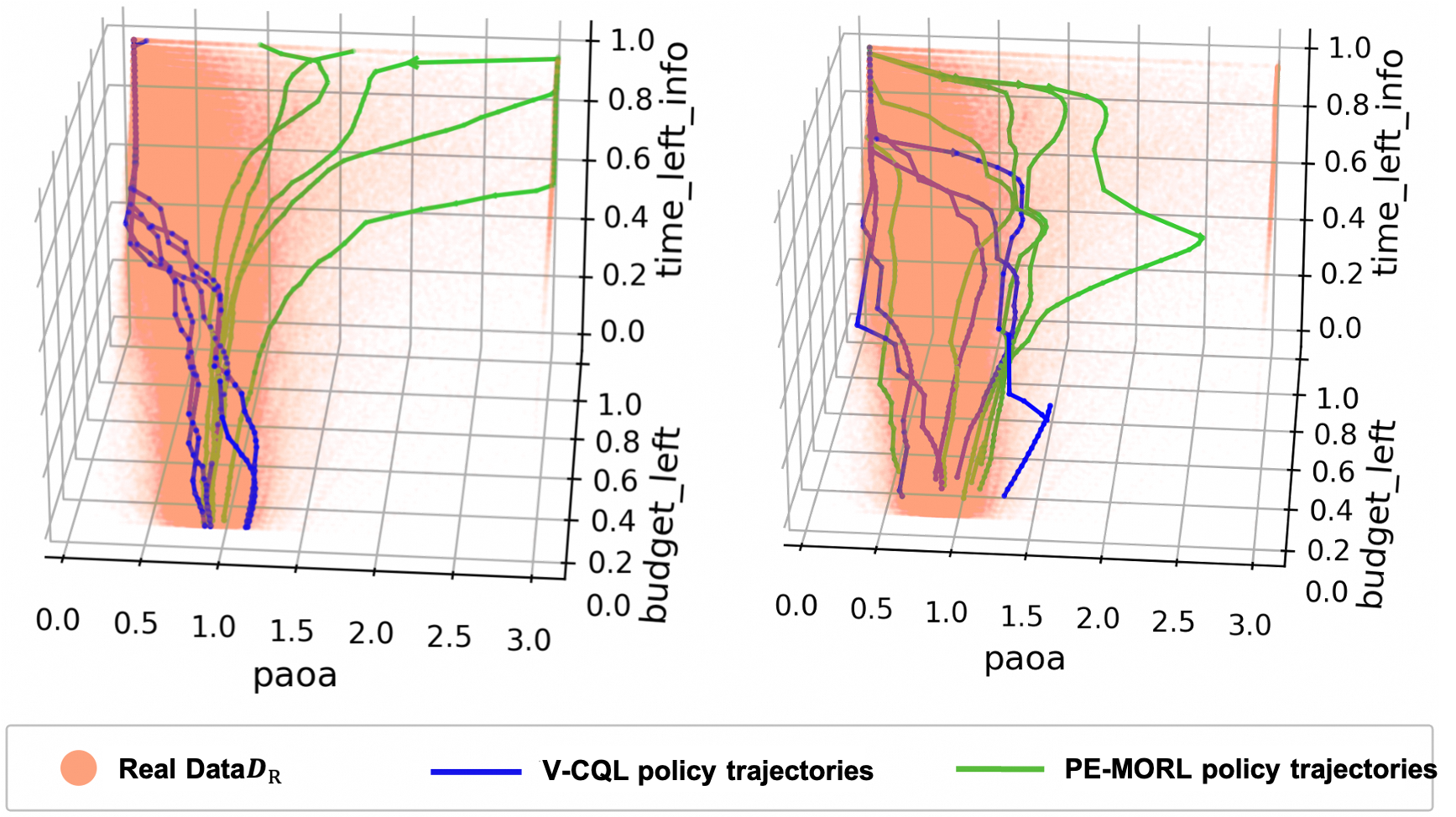}}
		\caption{Policy trajectories of two specific representative advertisers trained by V-CQL and PE-MORL algorithms.}
		\label{fig:distance_show_case}
	\end{center}
	\vskip -0.2in
\end{figure}

\begin{acks}
To Robert, for the bagels and explaining CMYK and color spaces.
\end{acks}

\bibliographystyle{ACM-Reference-Format}
\bibliography{model_based_RL_bidding}

\appendix
\section{Related Works}
\label{app:related_work}
\subsection{Mean Field Game Algorithms}

Mean field game (MFG) theory studies decision-making for a large population of homogeneous agents with small interactions. The basic idea of the MFG is to model the game of all agents as a game of two players intertwined with each other, including a representative agent and the empirical distribution of all other agents. The MFG algorithms usually look for the optimal policy in terms of social welfare under cooperative settings \cite{1,li2021permutation} or Nash Equilibrium policies under competitive/general-sum setting \cite{3}. For example, [1] designs a pessimistic value iteration method under the offline setting to maximize social welfare, and \cite{3} proposes a Q-learning-based algorithm with better convergence stabilities and learning accuracies. However, in this paper, we study the auto-bidding problem from the perspective of a single advertiser (representative advertiser) where other advertisers' (background advertisers) policies are assumed to be fixed and known. (Note that this assumption is viable in practice since the auto-bidding policy designer, usually the commercial department of an online advertising platform, can access the bidding policy of all advertisers, and there is no ethical issue since the observation of the considered advertiser's policy involves no information on other advertisers.)

\subsection{Offline Reinforcement Learning}

Offline RL (also known as batch RL) aims to learn better policies based on a fixed offline dataset collected by some behavior policies. The main challenge offline RL addressed is the extrapolation error (also known as the out-of-distribution, OOD) caused by missing data. The key technique to address the extrapolation error is to be pessimistic about the state-action pairs outside the offline dataset and conservatively train the policy. Similar to the traditional online RL algorithms, the offline RL algorithms can be divided into the model-free offline RL and the model-based offline RL.

Model-free Offline RL directly trains the policy with the offline dataset without building or learning an environment model. Based on the specific ways of addressing the extrapolation error, the model-free offline RL algorithms can be further divided into policy constraint methods such as BCQ \cite{5}, BEAR \cite{6}, and conservative regularization methods such as CQL \cite{kumar2020conservative}, as well as the constrained stationary distribution methods such as AlgaeDICE \cite{8} and OptiDice \cite{9}. Note that the CQL acts as a strong baseline and has been applied to real-world auto-bidding systems [10].

Model-based Offline RL, in contrast, first learns an environment model and then trains the policy with it. Note that model-free offline RL algorithms can only learn on the states in the offline dataset, leading to overly conservative algorithms. Nonetheless, the model-based offline RL algorithms have the potential for broader generalization by generating and training on additional imaginary data \cite{11}. The model-based offline algorithms can be divided into two categories. The first category is to build a pessimistic environment model based on the uncertainty estimations, such as MOPO \cite{yu2020mopo}, and MOReL \cite{13}. The second one is to apply pessimistic learning methods, e.g., conservative Q learning methods, when training with the learned environment model, such as the COMBO \cite{11} and the H2O \cite{14}. A key factor of model-based offline RL algorithms is the generalization ability of the learned environment model since the quality of the generated imaginary data largely determines the policy training. In this paper, we build a PE environment model, taking advantage of the generalization ability of the PE neural networks.

\subsection{Permutation Equivariance RL Algorithms}

A permutation equivariant (PE) function is a type of function that maintains the same output structure under the reordering of its inputs. Given a set of inputs, if we apply a permutation (rearrange the order of the inputs), a permutation equivariant function will produce a set of outputs that are permuted in the same way as the inputs. A permutation invariant (PI) function is a type of function that produces the same output regardless of the order of its inputs. In other words, the output of the function is unchanged under any permutation of the input elements. The PE/PI functions can be implemented by many structures, such as the DeepSet \cite{zaheer2017deep}, and a set of transformers \cite{18}, and have been used in many areas, such as RL, computer visions, and natrual language processings \cite{19,20}. However, most of the previous works employ PE/PI neural networks in modeling policies or value functions in RL algorithms \cite{16,17,21}. For example, \cite{16} uses the Deep Set to model the policies of the UAVs, taking advantage of the variable input dimensions and PI property of the Deep Set, and the policies are trained in a simulated environment implemented by the OpenAI Gym interface that can be freely interacted with. \cite{17} leverages the Deep Set and attention-based neural networks to encode the neighborhood information in the quadrotors' policies and value functions, and the algorithm is trained in a simulated environment that can be freely interacted with. In this paper, we utilize them to model the environment for the agent in RL algorithms, additionally taking advantage of their generalization abilities. Besides, we also highlight that

We are the first to prove the permutation equivariant/invariant (PE/PI) properties of the advertiser's transition rule and the reward function under a typical industrial auction mechanism that involves multiple stages (Proposition 5.1), which provides the basis for the PE environment model design.
We are the first to both theoretically (E.q. \eqref{equ:upper_bound_pe}) and empirically (Table 2 and Table 3) demonstrate the generalization superiority of PE/PI neural networks in modeling the environment in the auto-bidding field.

\section{Notations and Definitions}
\label{app:notations}
\textbf{Notations:}
$\mathbb{N}_+$ and $\mathbb{R}_+$ represent the positive integer set and the positive real scalar set, respectively;
$\cdot^\top$ represents the transpose operator;
$[N]$ denotes the positive integer set containing $1$ to $N$, where $N\in\mathbb{N}_+$;
$[\cdot]_+\triangleq\max\{\cdot, 0\}$; $|\mathcal{D}|$ represents the number of elements in the set $\mathcal{D}$; $\circ$ is the function composition operation; $\left\|\cdot\right\|_1$ denotes the 1-norm;
$\cup$ denotes the union operator between sets, and $\setminus$ denotes the subtraction operator between sets;

We next give some definitions that will be used in lateral proves.
Based on the permutation operator, we can define the \emph{orbit averaging} operator as follows.
\begin{definition}[Orbit Averaging]
	\label{def:orbit_averaging}
	An orbit averaging $\mathcal{Q}$ is an operator acting on functions $f:\mathcal{X}\rightarrow\mathcal{Y}$, i.e.,
	\begin{align}
		\mathcal{Q}f=\frac{1}{N!}\sum_{\rho\in\Omega_N}\rho^{-1}\circ f \circ \rho,
	\end{align}
	where $\mathcal{X}$ and $\mathcal{Y}$ are the spaces of $N$-dimensional ordered vectors.
\end{definition}

Moreover, as we need to analyze the generalization bound for the design environment model that is measured by the \emph{r-covering} number based on the $l_\infty$-norm \cite{duan2023equivariant}, we formally define them as follows.
\begin{definition}[$l_\infty$-Norm]
	The $l_\infty$ norm between two functions $f,g:\mathcal{X}\rightarrow\mathcal{Y}$ is defined as $\|f-g\|_\infty\triangleq\max_{\mathbf{x}\in\mathcal{X}}\left\|f(\mathbf{x})-g(\mathbf{x})\right\|_1$.
\end{definition}
\begin{definition}[R-covering Number]
	\label{def:r_cover}
	Given a function set $\mathcal{F}$, the function set $\mathcal{F}'$ is said to r-covers $\mathcal{F}$ under $l_\infty$ norm if for any function $f\in\mathcal{F}$, there exists $f'\in\mathcal{F}'$, such that $l_\infty(f,f')\le r$. The $r$-covering number of a function set $\mathcal{F}$, denoted as $\mathcal{N}_\infty(\mathcal{F},r)$, is the cardinality of the smallest function set that $r$-covers $\mathcal{F}$ under $l_\infty$ norm.
\end{definition}

\section{Proof of Lemmas}

\subsection{Proof of Lemma \ref{lemma:orbit_averaging}}
\begin{lemma}
The orbit averaging $\mathcal{Q}$ can make any function $f:\mathcal{X}\rightarrow\mathcal{Y}$ a PE function, i.e., 
$\rho(\mathcal{Q}f(\mathbf{x}))=\mathcal{Q}f(\rho\mathbf{x})$, for any $\mathbf{x}\in\mathcal{X}$. Moreover, if $f$ is a PE function, then $\mathcal{Q}f=f$.
\end{lemma}
\begin{proof}
    Following the definition of the orbit averaging in Definition \ref{def:orbit_averaging}, we have:
    \begin{align}
        \mathcal{Q}f(\rho\mathbf{x})&=\frac{1}{N!}\sum_{\rho_1\in\Omega_N}\rho_1^{-1}f(\rho_1\rho\mathbf{x})\notag\\
        &=\rho\bigg(\frac{1}{N!}\sum_{\rho_1\in\Omega_N}(\rho_1\rho)^{-1}f(\rho_1\rho\mathbf{x})\bigg)\notag\\
        &=\rho(\mathcal{Q}f(\mathbf{x})).
    \end{align}
    Note that $\rho_1\rho$ is also a permutation operator in $\Omega_N$, and when iterating $\rho_1$ through $\Omega_N$, $\rho_1\rho$ will also iterate through $\Omega_N$ without repetition.
    Besides, if $f$ is already a PE function, we have:
    \begin{align}
        \mathcal{Q}f(\mathbf{x})=\frac{1}{N!}\sum_{\rho\in\Omega_N}\rho^{-1}f(\rho\mathbf{x})=\frac{1}{N!}\sum_{\rho\in\Omega_N}\rho^{-1}\rho f(\mathbf{x})=f(\mathbf{x}).
    \end{align}
    This completes the proof.
\end{proof}
\label{app:proof_lemma_oa}

\subsection{Lemmas on Permutation Operators}
Recall that the permutation operator $\rho$ can swap elements in an ordered vector $\mathbf{a}\in\mathbb{R}^N$, i.e.,  
\begin{align}
	\rho\mathbf{a}=\rho(a_1,a_2,\cdots,a_N)=(a_{\rho^{-1}(1)},a_{\rho^{-1}(2)},\cdots,a_{\rho^{-1}(N)}).
\end{align}
As a matrix $\mathbf{A}\in\mathbb{R}^{N\times N}$ can be viewed as an order vector composed of the row vectors $\mathbf{a}_{r,j}\in\mathbb{R}^{1\times N}, i,j\in [N]$, i.e., $\mathbf{A}=(\mathbf{a}_{r,1:N}^\top)$, 
we can define the \emph{row permutation operator} $\rho_r$ on $\mathbf{A}$ as swapping its row vectors following $\rho$, i.e.,
\begin{align}
	\rho_r \mathbf{A}\triangleq \rho(\mathbf{a}_{r,1:N}^\top)=\begin{bmatrix} a_{1\rho^{-1}(1)} & a_{1\rho^{-1}(2)}  &\cdots & a_{1\rho^{-1}(N)} 
		\\ a_{2\rho^{-1}(1)} & a_{2\rho^{-1}(2)} & \cdots& a_{2\rho^{-1}(N)}
		\\
		\vdots&\vdots&\ddots&\vdots
		\\
		a_{N\rho^{-1}(1)} & a_{N\rho^{-1}(2)} & \cdots& a_{N\rho^{-1}(N)}
	\end{bmatrix},
\end{align}
where $a_{ij}$ denotes the element in the $i$-th row and the $j$-th column of $\mathbf{A}$.
Similarly, we can define the \emph{column permutation operator} $\rho_c$ on $\mathbf{A}$ as swapping its column vectors $\mathbf{a}_{c,i}\in\mathbb{R}^{N\times 1}$ following $\rho$, i.e.,
\begin{align}
	\label{equ:c_swap}
	 \mathbf{A}\rho_c\triangleq \rho(\mathbf{a}_{c,1:N})^\top=\begin{bmatrix} a_{\rho^{-1}(1)1} & a_{\rho^{-1}(1)2}  &\cdots & a_{\rho^{-1}(1)N} 
	 	\\ a_{\rho^{-1}(2)1} & a_{\rho^{-1}(2)2} & \cdots& a_{\rho^{-1}(2)N}
	 	\\
	 	\vdots&\vdots&\ddots&\vdots
	 	\\
	 	a_{\rho^{-1}(N)1} & a_{\rho^{-1}(N)2} & \cdots& a_{\rho^{-1}(N)N}
	 \end{bmatrix}.
\end{align}
Note that $\rho_c$ and $\rho_r$ are usually implemented by the \textbf{Permutation Matrix}, where every row and column has exactly a coefficient $1$ and the rest coefficients are all $0$'s.
Hence, we view $\rho_r$ and $\rho_c$ as matrices in the following. 
We next provide some lemmas that will be used later.
\begin{lemma}
	\label{lemma:i}
	It holds that $\rho_c^\top\rho_c=\mathbf{I}$ and $\rho_r^\top\rho_r=\mathbf{I}$.
\end{lemma}
\begin{proof}
There are many ways to prove this well-known lemma. A quick way to prove it is that: $\rho_c$ and $\rho_r$ both belong to the orthogonal matrix whose inverse is just its transpose.
\end{proof}
\begin{lemma}
	\label{lemma:equv}
	It holds that for any square matrix $\mathbf{A}$, we have $\mathrm{diag}(\rho_c^\top\mathbf{A}\rho_c)=\rho\mathrm{diag}(\mathbf{A})$, and $\mathrm{diag}(\rho_r\mathbf{A}\rho_r^\top)=\rho\mathrm{diag}(\mathbf{A})$.
\end{lemma}
\begin{proof}
	We first investigate what operations does $\rho_c^\top$ do. Specifically, from Lemma \ref{lemma:i}, we have:
	\begin{align}
		\rho_c^\top\rho_c=\rho_c^\top\mathbf{I}\rho_c=\rho_c^\top
		\begin{bmatrix} 
			\mathbf{1}^\top_{\rho(1)}\\
			\mathbf{1}^\top_{\rho(2)}\\
			\vdots\\
			\mathbf{1}^\top_{\rho(N)}\\
		\end{bmatrix}
		=\mathbf{I}=
			\begin{bmatrix} 
			\mathbf{1}^\top_{\rho^{-1}(\rho(1))}\\
			\mathbf{1}^\top_{\rho^{-1}(\rho(2))}\\
			\vdots\\
			\mathbf{1}^\top_{\rho^{-1}(\rho(N))}\\
		\end{bmatrix},
	\end{align}
where $\mathbf{1}_{\rho(i)}\in\mathbb{R}^{N\times 1}$ refers to the vector with one element $1$ at position $\rho(i)$ and zero elements at other  positions. 
Hence, we can know that $\rho_c^\top$ swaps the row vectors of the matrix according to $\rho$.
Then, we have 
	\begin{align}
		\mathrm{diag}(\rho_c^\top\mathbf{A}\rho_c)&=	\mathrm{diag}\bigg(\rho_c^\top\begin{bmatrix} a_{1\rho^{-1}(1)} & a_{1\rho^{-1}(2)}  &\cdots & a_{1\rho^{-1}(N)} 
			 \\ a_{2\rho^{-1}(1)} & a_{2\rho^{-1}(2)} & \cdots& a_{2\rho^{-1}(N)}
			 \\
			 \vdots&\vdots&\ddots&\vdots
			 \\
			 a_{N\rho^{-1}(1)} & a_{N\rho^{-1}(2)} & \cdots& a_{N\rho^{-1}(N)}
		  \end{bmatrix}\bigg)
	  \notag\\
	  &=(a_{\rho^{-1}(1)\rho^{-1}(1)} , a_{\rho^{-1}(2)\rho^{-1}(2)} ,\cdots, a_{\rho^{-1}(N)\rho^{-1}(N)} )\notag\\
	  &=\rho\mathrm{diag}(\mathbf{A}).
	\end{align}
Now, we prove $\mathrm{diag}(\rho_c^\top\mathbf{A}\rho_c)=\rho\mathrm{diag}(\mathbf{A})$, and 
$\mathrm{diag}(\rho_r\mathbf{A}\rho_r^\top)=\rho\mathrm{diag}(\mathbf{A})$ can be proved in the same way.
\end{proof}
\begin{figure*}[t]
	\begin{center}
		{\includegraphics[width=150mm]{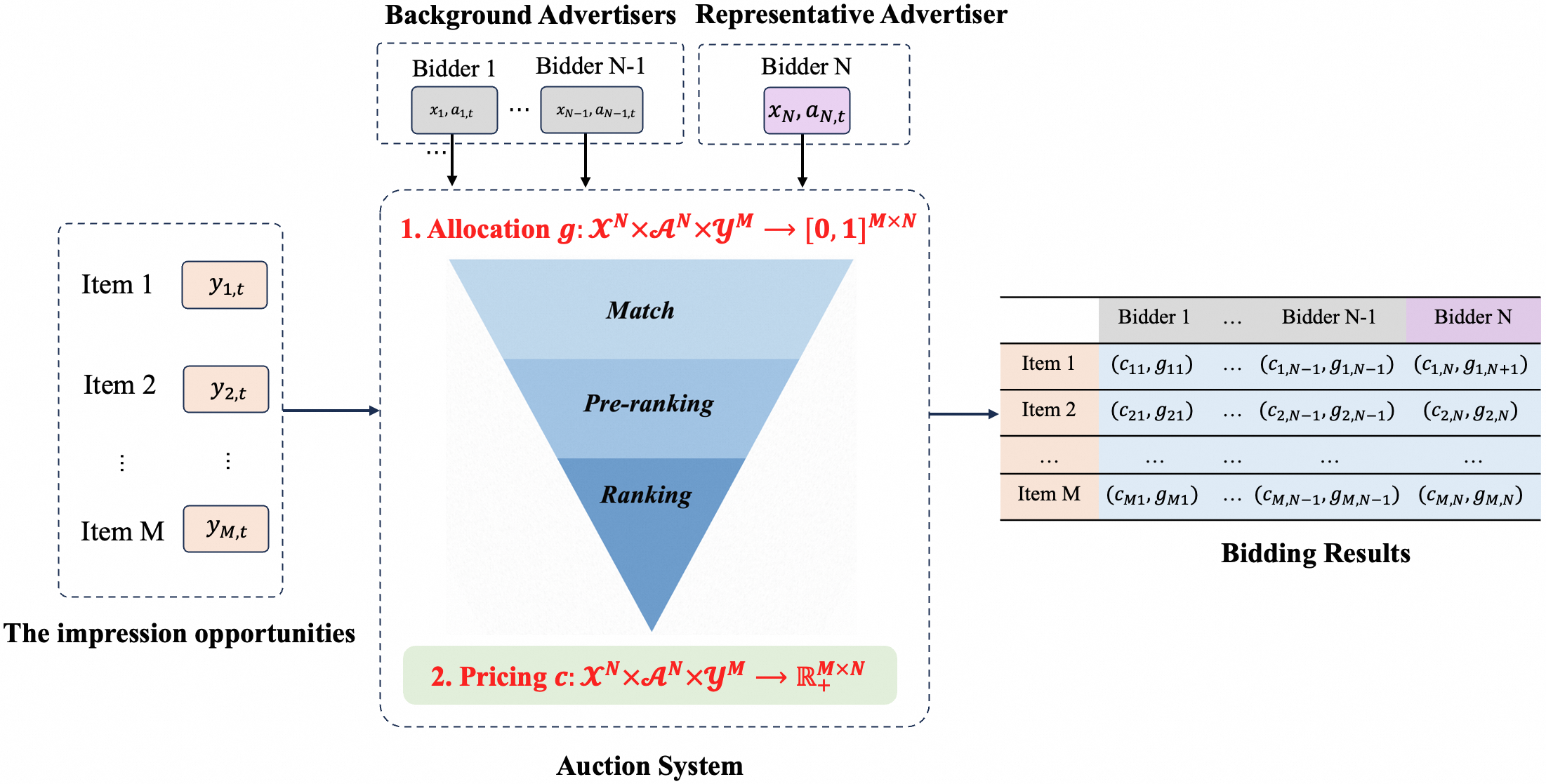}}
		\caption{An illustration of the bidding process in an industrial advertising system between time step $t$ and $t+1$. Specifically, there are $N-1$ background advertisers and one representative advertiser bidding for $M$ impression opportunities, where $x_i\in\mathcal{X}$ denotes the contextual feature of advertiser $i$, $a_{i,t}$ is the bid of advertiser $i$, and $y_i\in\mathcal{Y}$ denotes the contextual feature of the $j$-th impression opportunity, $j\in [M]$. The auction system can be viewed as two mechanisms, including the allocation mechanism $g:\mathcal{X}^N\times\mathcal{A}^N\times\mathcal{Y}^M\rightarrow [0,1]^{M\times N}$, and the pricing mechanism $c:\mathcal{X}^N\times\mathcal{A}^N\times\mathcal{Y}^M\rightarrow \mathbb{R}_+^{M\times N}$. The allocation mechanism $g$ determines the distribution of the impression opportunities among the advertisers, while the pricing mechanism $c$ decides how much each advertiser should pay.}
		\label{fig:ads_sys}
	\end{center}
\end{figure*}
\subsection{Other Lemmas}
 \begin{lemma}[Lemma B.2 in \cite{duan2023equivariant}]
 \label{lemma:generalization}
        If $|l(\cdot)|\le c$ for constant $c>0$ and $\forall f,f'\in\mathcal{F}, |l(f,u)-l(f',u)|\le K\|f-f'\|_\infty$, then we have 
        \begin{align}
              \Delta\le 2\inf_{r>0}\bigg\{\sqrt{\frac{2\ln\mathcal{N}_\infty(\mathcal{M}',r)}{|\mathcal{S}|}}+Kr\bigg\}+C_{\delta, \mathcal{S}},
        \end{align}
        where $\Delta\triangleq\mathbf{E}_{u\sim\mathcal{D}}[l(f,u)]-\frac{1}{m}\sum_{u\in\mathcal{S}}l(f,u)$, and $C_{\delta, \mathcal{S}}\triangleq4\sqrt{\frac{2\ln (4/\delta)}{|\mathcal{S}|}}$.
        Here $\mathcal{D}$ and $\mathcal{S}$ represent the real distribution of $u$ and the training set of $u$, and $m\triangleq |\mathcal{S}|$.
    \end{lemma}
    \begin{proof}
        See the proof of Lemma B.2 in \cite{duan2023equivariant}.
    \end{proof}

\begin{lemma}[Lemma B.6 in \cite{duan2023equivariant}]
\label{lemma:orbit_averaging}
    For function class $\mathcal{M}$ and orbit averaging operator $\mathcal{Q}$, if $\forall M_1,M_2\in\mathcal{M}$, we have $l_\infty(\mathcal{Q}M_1,\mathcal{Q}M_2)\le l_\infty(M_1,M_2)$, then 
    $\mathcal{N}_\infty(\mathcal{Q}\mathcal{M},r)\le\mathcal{N}_\infty(\mathcal{M},r)$ for any $r>0$.
\end{lemma}
\begin{proof}
    For any $r>0$, denote $\mathcal{M}_r$ as the smallest $r$-covering set that covers $\mathcal{M}$ with size $\mathcal{N}_\infty(\mathcal{M},r)$. For any $M\in\mathcal{M}$, let $M_r\in\mathcal{M}_r$ be the function that $r$-covers $M$. We have $l_\infty(\mathcal{Q}M_r,\mathcal{Q}M)\le l_\infty(M_r,M)\le r$. Therefore, $\mathcal{Q}\mathcal{M}_r$ is a $r$-covering set of $\mathcal{Q}\mathcal{M}$, and we have $  \mathcal{N}_\infty(\mathcal{Q}\mathcal{M},r)\le |\mathcal{Q}\mathcal{M}_r|\le |\mathcal{M}_r|=\mathcal{N}_\infty(\mathcal{M},r)$.
\end{proof}

\begin{lemma}[Lemma 4.3 in \cite{luo2018algorithmic}]
\label{lemma:telescope}
    Let the reward function of two environment models $M$ and $\hat{M}$ be the same, denoted as $r$, and let their transition rules be $P$ and $\hat{P}$, respectively. Then, the performance difference of a policy $\pi$ in two environment models satisfies:
    \begin{align}
        \eta_{\hat{M}}(\pi)-\eta_M(\pi)=\gamma\mathbb{E}_{(s,a)\sim d^\pi_{\hat{P}}}[G^\pi_{\hat{M}}(s,a)],
    \end{align}
    where $d^\pi_{\hat{P}}$ represents the state-action pair distributions under policy $\pi$ in $\hat{M}$, and $G^\pi_{\hat{M}}(s,a)=\sum_{s'}(\hat{P}(s,a)-P(s,a))V^\pi_M(s')$.
\end{lemma}
\begin{proof}
    See proof of the Lemma 4.3 in \cite{luo2018algorithmic}.
\end{proof}

\section{Proof of Theorems}
\label{app:proof_thm}
\begin{theorem}[Better Generalization Ability of the PE Environment Model]
	For any environment model $\hat{M}'$, making it satisfy  \eqref{equ:design_of_model} can reduce the upper bound of the gap $\Delta_G(\hat{M}')$, i.e.,
	\begin{align}
		\text{Upper Bound of }  \Delta_G(\hat{M})\le\text{Upper Bound of }  \Delta_G(\hat{M}'),
	\end{align}
	where $\hat{M}$ denotes the environment model $\hat{M}'$ after making it satisfy \eqref{equ:design_of_model}.
\end{theorem}
\begin{proof}
	For the sake of explanation, we express the outputs of the PE environment model $\hat{M}$ as $\hat{\mathbf{u}}_{t}\triangleq\hat{u}_{0:N,t}$ in this subsection, where $\hat{u}_{i,t}=[\hat{s}_{i,t}, \hat{r}_t]^\top$, and the design in \eqref{equ:design_of_model} can be equivalently written as:
	\begin{align}
		\label{equ:equivariant_design_sim}
		\rho\hat{\mathbf{u}}_t=\hat{M}(\rho\mathbf{s}_t,\rho\mathbf{a}_t).
	\end{align}
	Consider a random environment model $\hat{M}'$ without the design of \eqref{equ:equivariant_design_sim}.
	A common way to make $\hat{M}'$ satisfy \eqref{equ:equivariant_design_sim} is to apply the {orbit averaging} $\mathcal{Q}$ (as defined in Definition \ref{def:orbit_averaging}) to it, i.e.,
	\begin{align}
		\mathcal{Q}\hat{M}'=\frac{1}{N!}\sum_{\rho\in\Omega_{N}}\rho^{-1}\circ \hat{M}' \circ \rho.
	\end{align}
	This is because the orbit averaging $\mathcal{Q}$ can make any function a PE function. Formally, we have:
	\begin{lemma}
		\label{lemma:orbit_averaging}
		The orbit averaging $\mathcal{Q}$ can make any function $f:\mathcal{X}\rightarrow\mathcal{Y}$ a PE function, i.e., 
		$\rho(\mathcal{Q}f(\mathbf{x}))=\mathcal{Q}f(\rho\mathbf{x})$, for any $\mathbf{x}\in\mathcal{X}$. Moreover, if $f$ is a PE function, then $\mathcal{Q}f=f$.
	\end{lemma}
	The proof is given in Appendix \ref{app:proof_lemma_oa}.
	We consider the gap between the mean absolute error (MAE) of $\hat{M}'$ on the training dataset, denoted as $\mathcal{S}$, and that on the testing dataset, denoted as $\mathcal{D}$ \footnote{Theoretically, the testing dataset should be the real distribution of the state-action pairs $(\mathbf{s}_t, \mathbf{a}_t)$.}, as the generalization performance metric of the environment model, i.e., 
	\begin{align}
		\Delta_G(\hat{M}')\triangleq\mathbb{E}_{\mathcal{D}}[L(\hat{M}', \mathbf{s_t}, \mathbf{a}_t)]-\frac{1}{|\mathcal{S}|}\sum_{\mathcal{S}}L(\hat{M}', \mathbf{s_t}, \mathbf{a}_t),
	\end{align}
	where $L(\hat{M}', \mathbf{s_t}, \mathbf{a}_t)\triangleq\left\|\mathbf{u}_t-\hat{\mathbf{u}}_t\right\|_1$ is the MAE under $\mathbf{s}_t$ and $\mathbf{a}_t$, and $\mathbf{u}_t$ denotes the ground-truth.
	Next, we provide the upper bound of $\Delta_G(\hat{M}')$ in the following proposition.
	\begin{proposition}[Generalization Bound]
		\label{prop:generalization_bound}
		Let $\mathcal{M}'$ be the function class of $\hat{M}'$, with probability at least $\delta>0$, we have:
		\begin{align}
			\label{equ:generalization_bound}
			\Delta_G(\hat{M}')\le 2\inf_{r>0}\bigg\{\sqrt{\frac{2\ln\mathcal{N}_\infty(\mathcal{M}',r)}{|\mathcal{S}|}}+r\bigg\}+C_{\delta, \mathcal{S}},
		\end{align}
		where $C_{\delta, \mathcal{S}}\triangleq4\sqrt{\frac{2\ln (4/\delta)}{|\mathcal{S}|}}$, and $\mathcal{N}_\infty(\mathcal{M}',r)$ is the $r$-covering number (defined in Definition \ref{def:r_cover}) of $\mathcal{M}'$.
	\end{proposition}
	The proof is given in Appendix \ref{app:proof_prop_generalization_bound}.
	From \eqref{equ:generalization_bound} we know that the generalization bound $\Delta_G(\hat{M}')$ is positively correlated with the $r$-covering number of $\mathcal{M}'$, i.e., $\mathcal{N}_\infty(\mathcal{M}',r)$.
	After applying the orbit averaging to $\hat{M}'$, we 
	get a new environment model $\hat{M}''=\mathcal{Q}\hat{M}'$ belonging to a new function class $\mathcal{M}''\triangleq\mathcal{Q}\mathcal{M}'$, where $\mathcal{Q}\mathcal{M}'$ represents the function class obtained by applying the orbit averaging $\mathcal{Q}$ to every element in $\mathcal{M}'$.
	Notably, 
	it can be proved that $\mathcal{M}''$ has a smaller $r$-covering number than $\mathcal{M}'$, which thereby decreases the generalization bound $\Delta_G$.
	\begin{proposition}[Smaller $r$-covering Number]
		\label{prop:small_cover_number}
		Let $\mathcal{M}''\triangleq\mathcal{Q}\mathcal{M}'$.
		Then we have $\mathcal{N}_\infty(\mathcal{M}'',r)\le \mathcal{N}_\infty(\mathcal{M}',r)$.
	\end{proposition}
	Essentially, the orbit averaging $\mathcal{Q}$ is a non-expanding operator that can shrink the original function class $\mathcal{M}'$, making the new function class $\mathcal{M}''$ a smaller $r$-covering number.
	A formal proof is given in Appendix
	\ref{app:proof_prop_cover_number}.
	Now, we can claim that for a random environment model $\hat{M}'$, making it satisfy \eqref{equ:equivariant_design_sim} can reduce the upper bound of the gap between the MAE on the training dataset and the testing dataset, i.e.,
	\begin{align}
		\text{Upper Bound of }  \Delta_G(\hat{M}'')\le\text{Upper Bound of }  \Delta_G(\hat{M}')\,
	\end{align}
	can increase its generalization ability.
\end{proof}

\section{Proof of Propositions}
\subsection{Proof of Proposition \ref{prop:homo}}
\label{app:proof_prop_homo}

\begin{proposition}[The PE and PI Properties of the POMDP, Proposition \ref{prop:homo} in the main paper]
In the POMDP of the representation advertiser, the transition rule $P$ is a PE function, i.e., $	
\rho\mathbf{s}_{t+1}\sim P(\rho\mathbf{s}_t,\rho\mathbf{a}_t)$, and 
the reward $R$ is a PI function, i.e., $R(\rho\mathbf{s}_t,\rho\mathbf{a}_{t})=R(\mathbf{s}_t,\mathbf{a}_{t})$, $\forall t, \rho$.
\end{proposition} 
\begin{proof}
Essentially, the characteristics of the auction system endow the POMDP with the PE and PI properties. 
Consider the bidding process between time step $t$ and $t+1$, as shown in Fig. \ref{fig:ads_sys}. Specifically, there are $N-1$ background advertisers and one representative advertiser bidding for $M\in\mathbb{N}_+$ impression opportunities, where $x_i\in\mathcal{X}$ denotes the contextual feature of advertiser $i$, and $y_i\in\mathcal{Y}$ denotes the contextual feature of the $j$-th impression opportunity, $j\in [M]$. Here $\mathcal{X}$ and $\mathcal{Y}$ denote the contextual feature space of the advertiser and the impression opportunity, respectively. The auction system can be viewed as two mechanisms, including the allocation mechanism $g:\mathcal{X}^N\times\mathcal{A}^N\times\mathcal{Y}^M\rightarrow [0,1]^{M\times N}$, and the pricing mechanism $c:\mathcal{X}^N\times\mathcal{A}^N\times\mathcal{Y}^M\rightarrow \mathbb{R}_+^{M\times N}$. The allocation mechanism $g$ determines the distribution of the impression opportunities among the advertisers and is usually composed of cascaded stages, including match stage, pre-ranking stage, ranking stage, etc., and the pricing mechanism $c$ decides how much each advertiser should pay.
Let $\mathbf{G}(\mathbf{x},\mathbf{a}_t,\mathbf{y})$ and $\mathbf{C}(\mathbf{x},\mathbf{a}_t,\mathbf{y})$ be the outputs of the allocation mechanism $g$ and the pricing mechanism $c$, respectively. Denote $g_{ji}$ and $c_{ji}$ as the element in the $j$-th row and the $i$-th column of $\mathbf{G}$ and $\mathbf{C}$, respectively, 
meaning the allocation decision and the corresponding cost of impression opportunity $j$ concerning advertiser $i$.  
Note that $\sum_i g_{ji}=1$.
An important characteristic of the allocation mechanism $g$ is that it is permutation equivariant. Specifically, let the joint contextual feature of $N$ advertisers and $M$ impression opportunities be $\mathbf{x}\triangleq(x_{1:N})$ and $\mathbf{y}\triangleq(y_{1:M})$, respectively, and for any permutation operator $\rho$ we have:
\begin{align}
		\label{equ:pe_g}
	\mathbf{G}(\rho\mathbf{x}, \rho\mathbf{a}_t,\mathbf{y})=\mathbf{G}(\mathbf{x},\mathbf{a}_t,\mathbf{y})\rho_c.
\end{align}
The reason why the allocation mechanism $g$ is permutation equivariant is that the order of the advertiser sequence is not important and cannot change the bidding results. 
Similarly, the pricing function $c$ is also permutation equivariant, i.e.,
\begin{align}
		\label{equ:pe_c}
		\mathbf{C}(\rho\mathbf{x}, \rho\mathbf{a}_t,\mathbf{y})=\mathbf{C}(\mathbf{x},\mathbf{a}_t,\mathbf{y})\rho_c.
\end{align}
 Let  $\mathbf{V}(\mathbf{x},\mathbf{y})\in\mathbb{R}_+^{M\times N}$ be the value matrix, where the element in the $j$-th row and the $i$-th column $v_{ji}\in\mathbb{R}_+$ is the value of impression opportunity $j$ with respect to advertiser $i$. We claim that the value matrix also satisfies the permutation equivariant property, i.e.,
 \begin{align}
 	\label{equ:pe_v}
 	\mathbf{V}(\rho\mathbf{x}, \mathbf{y})=\mathbf{V}(\mathbf{x},\mathbf{y})\rho_c.
 \end{align}
The reason is also that the order of the advertiser sequence is not important and cannot change the value evaluation of the impression opportunities. 
During the bidding process, there are only two key factors each advertiser $i$ mainly focuses on, i.e., the value of the impression opportunities it can gain, $\sum_{j\in [M]}g_{ji}v_{ji}$, and the cost it should pay, $\sum_{j\in [M]}c_{ji}$. The former one is exactly the reward $r_{i,t}$ advertiser $i$ receives, and the joint reward can be expressed as  $\mathbf{r}_t(\mathbf{s}_t,\mathbf{a}_t)=\mathrm{diag}(\mathbf{G}(\mathbf{x},\mathbf{a}_t,\mathbf{y})^\top\mathbf{V}(\mathbf{x},\mathbf{y}))$. 
Note that $x_i$ is contained in the local state $s_{i,t}$.
Hence, 
based on Lemma \ref{lemma:equv} and the permutation equivariant property \eqref{equ:pe_g} and \eqref{equ:pe_v}, we have
\begin{align}
	\label{equ:r_homo}
	\mathbf{r}_t(\rho\mathbf{s}_t,\rho\mathbf{a}_t)&=\mathrm{diag}\bigg(\mathbf{G}(\rho\mathbf{x},\rho\mathbf{a}_t,\mathbf{y})^\top\mathbf{V}(\rho\mathbf{x},\mathbf{y})\bigg)\notag\\
	&=\mathrm{diag}\bigg((\mathbf{G}(\mathbf{x},\mathbf{a}_t,\mathbf{y})\rho_c)^\top\mathbf{V}(\mathbf{x},\mathbf{y})\rho_c\bigg)\notag\\
	&=\mathrm{diag}\bigg(\rho_c^\top\mathbf{G}(\mathbf{x},\mathbf{a}_t,\mathbf{y})^\top\mathbf{V}(\mathbf{x},\mathbf{y})\rho_c\bigg)\notag\\
	&=\rho\mathrm{diag}\bigg(\mathbf{G}(\mathbf{x},\mathbf{a}_t,\mathbf{y})^\top\mathbf{V}(\mathbf{x},\mathbf{y})\bigg)\notag\\
	&=\rho 	\mathbf{r}_t(\mathbf{s}_t,\mathbf{a}_t).
\end{align}
As the reward of the representative advertiser is the $N$-th element in $\mathbf{r}_t(\mathbf{s}_t,\mathbf{a}_t)$, we have:
\begin{align}
\label{equ:PI}
R(\mathbf{s}_t,\mathbf{a}_t)=R(\rho\mathbf{s}_t,\rho\mathbf{a}_t).
\end{align}
Let $\mathbf{J}\in\mathbb{R}^{M\times M}$ be the matrix with all $1$ elements, and the joint cost can be expressed as $\mathbf{c}_t(\mathbf{s}_t,\mathbf{a}_t)=\mathrm{diag}(\mathbf{C}(\mathbf{x},\mathbf{a}_t,\mathbf{y})^\top\mathbf{J})$. Note that the change in state depends only on the cost spent between time step $t$ and $t+1$, i.e., 
$P(\mathbf{s}_{t+1}|\mathbf{s}_t,\mathbf{a}_t)=P(\mathbf{c}_t(\mathbf{s}_t,\mathbf{a}_t))$, where $\mathbf{c}_t$ is the corresponding cost during the state transition. Therefore, we have:
\begin{align}
	\label{equ:p_homo}
	P\bigg(\rho\mathbf{s}_{t+1}|\rho\mathbf{s}_t,\rho\mathbf{a}_t\bigg)&=P\bigg(\mathbf{c}_t(\rho\mathbf{s}_t,\rho\mathbf{a}_t)\bigg)\notag\\
	&=P\bigg(\mathrm{diag}\bigg(\mathbf{C}(\rho\mathbf{x},\rho\mathbf{a}_t,\mathbf{y})^\top\mathbf{J}\bigg)\bigg)\notag\\
	&=P\bigg(\mathrm{diag}\bigg(\rho_c^\top\mathbf{C}(\mathbf{x},\mathbf{a}_t,\mathbf{y})^\top\mathbf{J}\rho_c\bigg)\bigg)\notag\\
	&=P\bigg(\rho\mathrm{diag}\bigg(\mathbf{C}(\mathbf{x},\mathbf{a}_t,\mathbf{y})^\top\mathbf{J}\bigg)\bigg)\notag\\
	&=P\bigg(\rho\mathbf{c}_t(\mathbf{s}_t,\mathbf{a}_t) \bigg)\notag\\
	&=P\bigg(\mathbf{c}_t(\mathbf{s}_t,\mathbf{a}_t) \bigg)\notag\\
	&=P\bigg(\mathbf{s}_{t+1}|\mathbf{s}_t,\mathbf{a}_t\bigg).
\end{align}
Note that we use the properties $\mathbf{J}=\mathbf{J}\rho_c$ and $P(\rho\mathbf{c}_t(\mathbf{s}_t,\mathbf{a}_t))=P(\mathbf{c}_t(\mathbf{s}_t,\mathbf{a}_t))$.
So far, \eqref{equ:PI} and \eqref{equ:p_homo} have concluded the proof. 
\end{proof}

\subsection{Proof of Proposition \ref{prop:generalization_bound}}
\label{app:proof_prop_generalization_bound}
\begin{proposition}[Generalization Bound, Proposition \ref{prop:generalization_bound} in the main paper]
    Let $\mathcal{M}'$ be the function class of $\hat{M}'$, with probability at least $\delta$, we have:
    \begin{align}
        \Delta_G\le 2\inf_{r>0}\bigg\{\sqrt{\frac{2\ln\mathcal{N}_\infty(\mathcal{M}',r)}{|\mathcal{S}|}}+r\bigg\}+C_{\delta, \mathcal{S}},
    \end{align}
    where $C_{\delta, \mathcal{S}}\triangleq4\sqrt{\frac{2\ln (4/\delta)}{|\mathcal{S}|}}$, and $\mathcal{N}_\infty(\mathcal{M}',r)$ is the $r$-covering number of $\mathcal{M}'$ defined in Definition \ref{def:r_cover}.
\end{proposition}
\begin{proof}
    We first show that: for any two functions $\hat{M}_1$ and $\hat{M}_2$, the loss function satisfies:
    \begin{align}
        |L(\hat{M}_1, \mathbf{s}_t,\mathbf{a}_t)-L(\hat{M}_2, \mathbf{s}_t,\mathbf{a}_t)|\le \|\hat{M}_1-\hat{M}_2\|_\infty.
    \end{align} 
    Without the loss of generality, let $L(\hat{M}_1, \mathbf{s}_t,\mathbf{a}_t)>L(\hat{M}_2, \mathbf{s}_t,\mathbf{a}_t)$.
    Then, we have:
    \begin{align}
        &\;\;\;\;|L(\hat{M}_1, \mathbf{s}_t,\mathbf{a}_t)-L(\hat{M}_2, \mathbf{s}_t,\mathbf{a}_t)|\notag\\
        &=L(\hat{M}_1, \mathbf{s}_t,\mathbf{a}_t)-L(\hat{M}_2, \mathbf{s}_t,\mathbf{a}_t)\notag\\
        &=\|\mathbf{u}_t-\mathbf{u}_{1,t}\|_1-\|\mathbf{u}_t-\mathbf{u}_{2,t}\|_2\notag\\
        &=\sum_i |u_{i,t}-u_{i,1,t}|-\sum_i |u_{i,t}-u_{i,2,t}|\notag\\
        &=\sum_i |u_{i,t}-u_{i,1,t}+u_{i,2,t}-u_{i,2,t}|-\sum_i |u_{i,t}-u_{i,2,t}|\notag\\
        &\le\sum_i\bigg[|u_{i,t}-u_{i,2,t}|+|u_{i,2,t}-u_{i,1,t}|-|u_{i,t}-u_{i,2,t}|\bigg]\notag\\
        &=\sum_i|u_{i,1,t}-u_{i,2,t}|\notag\\
        &=\|\mathbf{u}_{1,t}-\mathbf{u}_{2,t}\|_1\notag\\
        &=\|\hat{M}_1(\mathbf{s}_t,\mathbf{a}_t)-\hat{M}_2(\mathbf{s}_t,\mathbf{a}_t)\|_1\notag\\
        &\le \max_{\mathbf{s},\mathbf{a}}\|\hat{M}_1(\mathbf{s},\mathbf{a})-\hat{M}_2(\mathbf{s},\mathbf{a})\|_1\notag\\
        &=\|\hat{M}_1-\hat{M}_2\|_\infty.
    \end{align}
Hence, we can leverage Lemma \ref{lemma:generalization} with $K=1$. This finishes the proof.
\end{proof}
\subsection{Proof of Proposition \ref{prop:small_cover_number}}
\label{app:proof_prop_cover_number}

\begin{proposition}[Smaller $r$-covering Number]
Let $\mathcal{M}''\triangleq\mathcal{Q}\mathcal{M}'$.
Then we have $\mathcal{N}_\infty(\mathcal{M}'',r)\le \mathcal{N}_\infty(\mathcal{M}',r)$.
\end{proposition}
\begin{proof}
    Consider two random functions $\hat{M}'_1,\hat{M}'_2\in\mathcal{M}'$, and let $\hat{M}''_1\triangleq\mathcal{Q}\hat{M}'_1\in\mathcal{M}''$ and $\hat{M}''_2\triangleq\mathcal{Q}\hat{M}'_2\in\mathcal{M}''$.
    With Lemma \ref{lemma:orbit_averaging}, we only need to prove $l_\infty(\hat{M}''_1, \hat{M}''_2)\le l_\infty(\hat{M}'_1, \hat{M}'_2)$. Specifically, 
    we have:
    \begin{align}
        l_\infty(\hat{M}''_1, \hat{M}''_2)&=\max_{\mathbf{s},\mathbf{a}}\|\hat{M}''_1(\mathbf{s},\mathbf{a})-\hat{M}''_2(\mathbf{s},\mathbf{a})\|_1\notag\\
&=\max_{\mathbf{s},\mathbf{a}}\|\mathcal{Q}\hat{M}'_1(\mathbf{s},\mathbf{a})-\mathcal{Q}\hat{M}'_2(\mathbf{s},\mathbf{a})\|_1\notag\\
        &=\max_{\mathbf{s},\mathbf{a}}\bigg\|\frac{1}{N!}\sum_{\rho\in\Omega_N}\rho^{-1}\bigg(\hat{M}'_1(\rho\mathbf{s},\rho\mathbf{a})-\hat{M}'_2(\rho\mathbf{s},\rho\mathbf{a})\bigg)\bigg\|_1\notag\\
        &\le \max_{\mathbf{s},\mathbf{a}}\frac{1}{N!}\sum_{\rho\in\Omega_N}\bigg\|\bigg(\hat{M}'_1(\rho\mathbf{s},\rho\mathbf{a})-\hat{M}'_2(\rho\mathbf{s},\rho\mathbf{a})\bigg)\bigg\|_1\notag\\
        &=\frac{1}{N!}\sum_{\rho\in\Omega_N}\max_{\mathbf{s},\mathbf{a}}\bigg\|\bigg(\hat{M}'_1(\rho\mathbf{s},\rho\mathbf{a})-\hat{M}'_2(\rho\mathbf{s},\rho\mathbf{a})\bigg)\bigg\|_1\notag\\
        &=\frac{1}{N!}\sum_{\rho\in\Omega_N}\max_{\mathbf{s},\mathbf{a}}\bigg\|\bigg(\hat{M}'_1(\mathbf{s},\mathbf{a})-\hat{M}'_2(\mathbf{s},\mathbf{a})\bigg)\bigg\|_1\notag\\
        &=\frac{1}{N!}\sum_{\rho\in\Omega_N} l_\infty(\hat{M}'_1, \hat{M}'_2)\notag\\
        &=l_\infty(\hat{M}'_1, \hat{M}'_2).
    \end{align}
    This concludes the proof.
\end{proof}

\subsection{Proof of Proposition \ref{prop:lower_bound}}
\label{app:proof_lower_bound}

\begin{proposition}[Lower Bound Improvement, Proposition \ref{prop:lower_bound} in the main paper]
    For any policy $\pi$, its performance in the online advertising system $M$, denoted as $\eta_M(\pi)$, is lower bounded by its performance in a fitted environment model $\hat{M}$ with the penalized predicted reward $\tilde{r}_t$, denoted as $\eta_{\hat{M}}(\pi)$, i.e., $\eta_M(\pi)\ge\eta_{\hat{M}}(\pi)$.
\end{proposition}
\begin{proof}
    With Lemma \ref{lemma:telescope}, we can get:
    \begin{align}
    \label{equ:eta_M}
\eta_M(\pi)&=\mathbb{E}_{(\mathbf{s},\mathbf{a})\sim d^\pi_{\hat{M}}}[R(\mathbf{s},\mathbf{a})-\gamma G^\pi_{\hat{M}}(\mathbf{s},\mathbf{a})]\notag\\
&\ge \mathbb{E}_{(\mathbf{s},\mathbf{a})\sim d^\pi_{\hat{M}}}[R(\mathbf{s},\mathbf{a})-\gamma |G^\pi_{\hat{M}}(\mathbf{s},\mathbf{a})|]\notag\\
&\ge \mathbb{E}_{(\mathbf{s},\mathbf{a})\sim d^\pi_{\hat{M}}}[R(\mathbf{s},\mathbf{a})-\lambda u(\mathbf{s},\mathbf{a})]\notag\\
&=\mathbb{E}_{(\mathbf{s},\mathbf{a})\sim d^\pi_{\hat{M}}}[\tilde{r}_t]\notag\\
&=\eta_{\hat{M}}(\pi).
    \end{align}
    Note that we utilize a property \cite{yu2020mopo} in the above derivation, i.e., there exists a function $u(\mathbf{s},\mathbf{a})$ such that 
    \begin{align}
    u(\mathbf{s},\mathbf{a})\ge d_\mathcal{F}\triangleq \sup{f\in\mathcal{F}}|\mathbb{E}_{\mathbf{s}'\sim\hat{P}}f(\mathbf{s'})-\mathbb{E}_{\mathbf{s}'\sim{P}}f(\mathbf{s'})| \ge |G^\pi_{\hat{M}}|,
    \end{align}
    where $\mathcal{F}$ denotes the function class of the value function $V^\pi_M(\mathbf{s})$. Hence, \eqref{equ:eta_M} concludes the proof.
\end{proof}

\end{document}